\documentclass[onefignum,onetabnum]{siamonline190516}
\usepackage[square,sort,comma,numbers]{natbib}
\usepackage{lipsum}
\usepackage{enumerate}
\usepackage{url} 

\usepackage{graphicx, color}
\usepackage{amsmath}
\usepackage{amsfonts}
\usepackage{color}
\usepackage{amsmath}
\usepackage{amssymb}
\newtheorem{remark}{Remark}

\DeclareMathOperator*{\argmax}{arg\,max}
\DeclareMathOperator*{\argmin}{arg\,min}

\def\blue#1{{#1}}

\begin{document}
	\title{Benefit of Interpolation in Nearest Neighbor Algorithms}
	
	\author{
		Yue Xing\thanks{Department of Statistics, Purdue University}
		\and
		Qifan Song\footnotemark[1]
		\and
		Guang Cheng\footnotemark[1]\textsuperscript{ }\thanks{Department of Statistics, University of California, Los Angeles}
	}
	
	

	
	
	
	\maketitle
	\begin{abstract}
		In some studies \citep[e.g.,][]{zhang2016understanding} of deep learning, it is observed that over-parametrized deep neural networks achieve a small testing error even when the training error is almost zero. Despite numerous works towards understanding this so-called ``double descent'' phenomenon  \citep[e.g.,][]{belkin2018reconciling,belkin2019two}, in this paper, we turn into another way to enforce zero training error (without over-parametrization) through a data interpolation mechanism. Specifically, we consider a class of interpolated weighting schemes in the nearest neighbors (NN) algorithms. By carefully characterizing the multiplicative constant in the statistical risk, we reveal a U-shaped performance curve for the level of data interpolation in both classification and regression setups. This sharpens the existing result \citep{belkin2018does} that zero training error does not necessarily jeopardize predictive performances and claims a counter-intuitive result that a mild degree of data interpolation actually {\em strictly} improve the prediction performance and statistical stability over those of the (un-interpolated) $k$-NN algorithm. In the end, the universality of our results, such as change of distance measure and corrupted testing data, will also be discussed. 
	\end{abstract}
	
	\section{Introduction}

	Statistical learning algorithms play a central role in modern data analysis and artificial intelligence. A supervised learning algorithm aims at constructing a predictor $h^*$, which uses training samples to predict testing data. Given a loss function and the class of candidate predictors $\mathcal{H}$,  the function $h^*\in\mathcal{H}$ is usually selected by minimizing the empirical loss. 
	To avoid over-fitting, classical learning theory \citep{anthony1995function,bartlett1996fat} suggests controlling the capacity of model space $\mathcal{H}$ (e.g., VC-dimension, fat-shattering dimension, Rademacher complexity), and discourages data interpolation.
	\blue{On the other hand, recent} deep learning applications reveal a completely different phenomenon from classical understanding, that is, with heavily over-parameterized neural network models, when the training loss reaches zero, the testing performance is still sound. \blue{For example, \cite{zhang2016understanding,sanyal2020benign,li2021towards} demonstrated experiments where deep neural networks have a small generalization error even when the training data are perfectly fitted and gave discussions towards \blue{this} phenomenon.}
	This counter-intuitive phenomenon motivates various theoretical studies to investigate the testing performance of estimators belonging to the ``over-fitting regime." For instance, \cite{arora2019fine,xie2016diverse,cao2019generalization,arora2018stronger,bartlett2017spectrally,neyshabur2017pac} established generalization bounds for shallow neural networks with a growing number of hidden nodes; \blue{ \cite{belkin2019two,hastie2019surprises,bartlett2020benign,cao2021risk,chatterji2021finite}  provided the theoretical characterization of the interpolated estimators in some particular models, e.g. linear regression and binary classification}. Some later studies also characterize how the learning curve changes in different models \cite{chen2020multiple}, or its universality beyond natural training \cite{min2020curious}. All these aforementioned results deliver proper explanations for the phenomenon that ``\blue{(proper)} over-fitting does not hurt prediction" under parametric modelings.

	Inspired by these studies, this paper aims to sharpen the existing results on the relationship between over-fitting and testing performance in a nonparametric estimation procedure and provide insights into the phenomenon, ``over-fitting does not hurt prediction", from a different perspective. \blue{ To be more specific, we study the interpolated nearest neighbors algorithm (interpolated-NN, \citep{belkin2018overfitting}), which is a nonparametric regression/classification estimator that interpolates the training data.} \cite{belkin2018overfitting,belkin2018does} have proved the rate optimality of interpolated-NN regression estimation. Beyond rate optimality claim, so far, there is no insight about the behavior of nearest neighbor estimations within the overfitting regimes, such as how different interpolation schemes affect the characteristics of \blue{interpolated-NN}.

	We conduct a comprehensive analysis to quantify the risk of interpolated-NN estimators, including its convergence rate and, more importantly, the associated multiplicative constant. The interpolated-NN estimator, $h_\gamma$, is indexed by a parameter $\gamma$, which directly represents the level of interpolation. As $\gamma$ increases, the value of $h_\gamma(x)$ is more influenced by data point value $Y^1(x)$ in the sense that $\lim_{\gamma\rightarrow\infty}h_\gamma(x)=Y^1(x)$, where $X^1(x)$ denotes the nearest neighbor of $x$ and $Y^1(x)$ is the response of $X^1(x)$. As the interpolation level $\gamma$ increases \blue{within a proper range}, the rate of convergence of the squared bias term and the variance term in the variance-bias decomposition are not affected. Under proper smoothness conditions, as $\gamma$ increases, the multiplicative constants associated with the squared bias and the variance terms decrease and increase, respectively. More importantly, when $\gamma$ is small, the decrease of squared bias dominates the growth of variance; hence mild level interpolation strictly improves interpolated-NN compared with non-interpolated $k$-NN estimator. Overall, the risk of interpolated-NN, as a function of interpolation level $\gamma$, is U-shaped. Therefore, within the over-fitting regime, the risk of interpolated NN will first decrease and then increase concerning the interpolation level. A graphic illustration can be found in Figure \ref{fig:my_label} below with data dimension $d=2$. Besides, we conduct a similar analysis to study the statistical stability \citep{sun2016stabilized} of interpolated-NN classifier and obtain similar results where the stability of interpolated-NN is U-shaped in the interpolation level as well.
	\begin{figure}[!ht]
		\centering
		\includegraphics[scale=0.22]{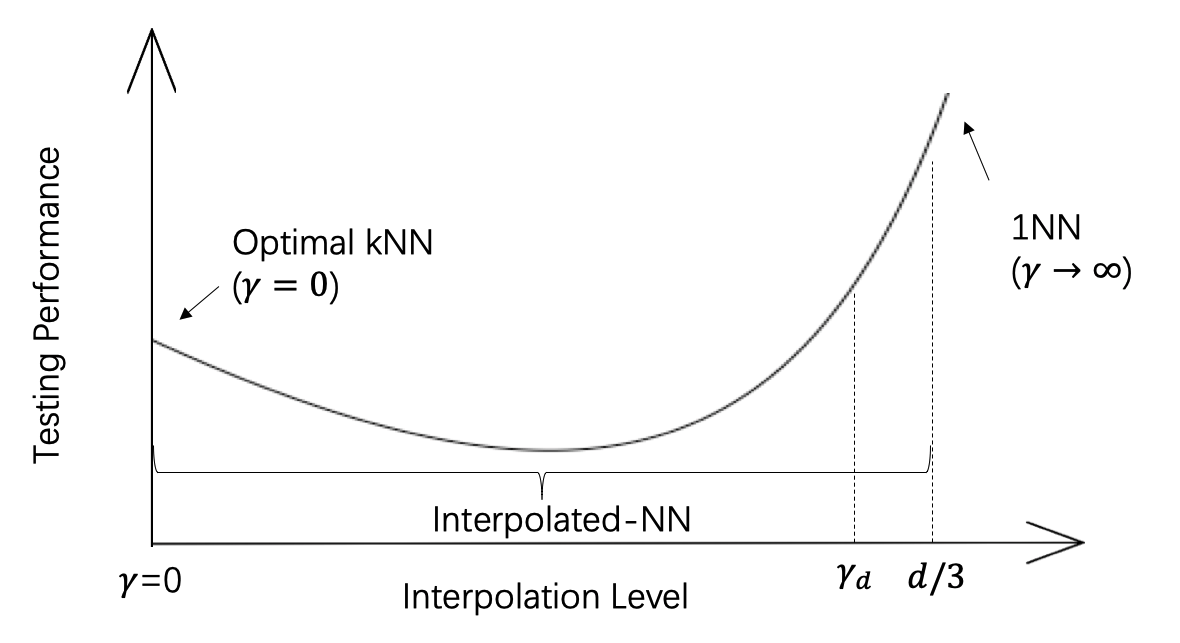}
		\caption{Relationship between testing performance and level of interpolation in interpolated-NN.}
		\label{fig:my_label}
	\end{figure}
	\begin{figure}[ht]
		\centering
		\includegraphics[scale=0.5]{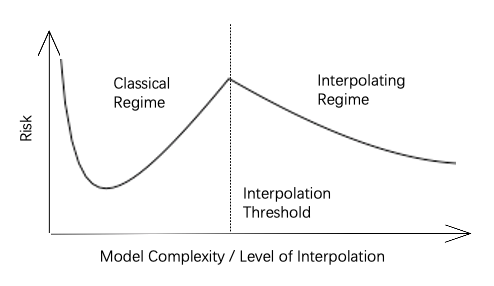}
		\caption{Double-descent Phenomenon in modern machine learning theory}
		\label{fig:double1}
	\end{figure}
	Our major finding, which in spirit claims that increasing degree of overfitting \blue{potentially} benefit testing performance, is somewhat similar to the recently discovered ``double-descent phenomenon'' \cite{belkin2018reconciling}. A graphical illustration of the ``double-descent phenomenon'' is shown in Figure \ref{fig:double1}. {In the first regime of the double-descent phenomenon, the ``classical regime", the model complexity is below some ``interpolation threshold" and the population loss is U-shaped when the model complexity increases. If the model complexity keeps increasing, in the second regime, ``overfitting regime", the training loss will be \blue{exactly} zero, and the generalization performance will get better in the model complexity.} \blue{In both the double-descent phenomenon and our result, the estimator gets improved with \blue{a} proper level of interpolation. However, besides the fact that our work focuses on nonparametric models while the double-descent phenomenon results are mostly on parametric models, it is noteworthy to emphasize the crucial differences between this double descent phenomenon and our result.} \blue{The literature of the double-descent phenomenon studies} the change of interpolated estimators' performance for {\em the level of over-parameterization}, while our work focuses on {\em the level of interpolation}. Although both are related to model over-fitting, they are not equivalent. The level of over-parameterization refers to the complexity of model space $\mathcal{H}$. In contrast, the level of interpolation is merely a measure of the estimating procedure (i.e., how much $h(x)$ is affected by the nearest neighbor of $x$) and has nothing to do with the complexity or dimensionality of $\mathcal{H}$. Hence, in the double-descent literature, one compares the risk of estimator $h(\cdot)$ across different dimensional setups, while in our work, the data dimension is fixed. It is worth noting that some recent works (e.g., \citealp{liu2020kernel}) discover the double-descent phenomenon for nonparametric regression as well (e.g., kernel ridge regression), which also considers that the dimension of model space increases. 
	
	Besides the main contributions, there are some minor contributions. First, we notice that the convergence rate given in \cite{belkin2018overfitting} for classification tasks is sub-optimal, and we improve it to optimal rate convergence \blue{via technical improvements}. Second, we study interpolated-NN under some other scenarios, e.g., change on distance metric and corruptions in testing data attributes. To be more specific, (1) We consider how the choice of distance metric used affects the asymptotic behavior of interpolated-NN. Traditional NN estimation usually utilizes $\mathcal{L}_2$ distance to determine the neighborhood set, and we will investigate whether or not the effect of interpolation is universal among different choices of distance metric. (2) Since interpolation-NN leads to a rather rugged and non-smooth regression estimation function, i.e., the nonparametric estimation function $h_\gamma(x)$ will have a big jump, especially around the training samples (refer to Figure \ref{toy} in Section \ref{sec:knn_wnn_compare}), we will also investigate in whether interpolation affects the prediction performance when testing data is slightly perturbed.
	
	Finally, we want to emphasize that this paper does not aim to promote the practical use of this interpolation method, given that $k$-NN is more user-friendly. Instead, our study on the interpolated-NN algorithm is used to describe the role of interpolation in generalization ability precisely, which is a new and interesting phenomenon motivated by the deep neural network practices. 
	
	The structure of the paper is as follows: in Section \ref{sec:model}, we introduce interpolated-NN and the level of interpolation in this method. We further present theorems regarding to rate optimality and multiplicative constant in Section \ref{sec:rate} and \ref{sec:const} respectively. A measure of statistical instability is also considered in Section \ref{sec:rate} and \ref{sec:const}. In Section \ref{sec:const}, in addition to the main theorems, some discussions relating to double descent, or more specifically, about the interpolation regime, will be given in Section \ref{sec:appendix:metric}, followed by numerical experiments in Section \ref{sec:numerical} and conclusion in Section \ref{sec:conclusion}.
	
	\section{Interpolation in Nearest Neighbors Algorithm}\label{sec:model}
	In this section, we first review the notations and the detailed algorithm of interpolated-NN. 
	Denote $(X_i,Y_i)$ as training samples for $i=1,...,n$. Given a data point $x$ for testing, define $R_{k+1}(x)$ to be the distance between $x$ and its $(k+1)$th nearest neighbor. Without loss of generality, let $X^1(x),\ldots, X^k(x)$ denote the {\it unsorted} $k$ nearest neighbors of $x$, and write as $X^1,\ldots, X^k$ for simplicity. Denote $\{R_i(x)\}_{i=1}^k$ and $\{Y^i(x)\}_{i=1}^k$ ($\{R_i\}_{i=1}^k$ and $\{Y^i\}_{i=1}^k$ in short) as distances between $x$ and $X^i$ \blue{and} the corresponding responses.
	
	For regression models, the aim is to estimate $\mathbb{E}(Y|X=x)$. Denote $\eta(x)$ as the target function, where $\eta(x)=\mathbb{E}(Y|X=x)$, and $\sigma^2(x)=Var(Y|X=x)$. A NN regression estimator at $x$, based on nearest $k$ neighbors, is defined as
	\begin{equation}\label{weightedNN}
		\widehat{\eta}_{k,n}(x)=\sum_{i=1}^k W_i(x) Y^i,
	\end{equation}
	where $W_i(x)$ (write as $W_i$ later) represents the (data-dependent) weight for each neighbor. The weighting scheme satisfies $W_i\geq 0$ and $\sum_{i=1}^k W_i=1$. For traditional $k$-NN, $W_i=1/k$ for all $i=1,...,k$. 
	
	For binary classification, denote $\eta(x)=P(Y=1|X=x)$, with $g(x)=1\{\eta(x)>1/2\}$ as the Bayes classifier. The interpolated-NN classifier is further defined as
	\begin{eqnarray*}
		\widehat{g}_{k,n}(x)=\begin{cases}
			1\qquad \sum_{i=1}^k W_iY^i>1/2\\
			0\qquad \sum_{i=1}^k W_iY^i\leq 1/2
		\end{cases}.
	\end{eqnarray*}
	The asymptotics of traditional $k$-NN algorithm have been extensively studied in \blue{the} literature (e.g., \citealp{cover1968rates,wagner1971convergence,fritz1975distribution,schoenmakers2013optimal,sun2016stabilized,CD14,xing2021predictive}). Variants of $k$-NN have been proposed to improve the  convergence, stability or computational performance, e.g.,optimally weighted NN \citep{samworth2012optimal}, locally weighted NN \citep{cannings2017local}, stabilized NN \citep{sun2016stabilized}, distributed NN \citep{duan2018DNN}, pre-processed 1NN \citep{xue2017achieving}, and robust NN \citep{wang2017analyzing}.

	An interpolating algorithm will interpolate training data, i.e., the fitted value/label is the same as the response/label of the training data. From this aspect, traditional $k$-NN does not interpolate unless $k=1$. To incorporate interpolation with NN algorithm, we follow \cite{belkin2018overfitting} to adopt the interpolated weighting scheme below: 		\begin{eqnarray}\label{eqn:model}
		W_i=\frac{R_i^{-\gamma}}{\sum_{j=1}^k R_j^{-\gamma}}=\frac{(R_i/R_{k+1})^{-\gamma}}{\sum_{j=1}^k (R_j/R_{k+1})^{-\gamma}}:=\frac{ \phi(R_i/R_{k+1}) }{\sum_{j=1}^k\phi(R_j/R_{k+1})},
	\end{eqnarray}
	for $i=1,\ldots,k$ and some $\gamma\geq 0$. Rewrite $\widehat{\eta}_{k,n}$ and $\widehat{g}_{k,n}$ as $\widehat{\eta}_{k,n,\gamma}$ and $\widehat{g}_{k,n,\gamma}$ to emphasize the choice of $\gamma$. In interpolated-NN, the function $\phi$ is chosen as $\phi(t)=t^{-\gamma}$. In general, if the positive function $\phi$ satisfies $\lim_{t\rightarrow 0}\phi(t)=\infty$ and $\phi$ decreases in $t$, through replacing the $\phi$ in $W_i$, we can create different interpolating weighting schemes.
	
	The parameter $\gamma\geq0$ controls the level of interpolation: with a larger $\gamma>0$, the algorithm will put more weights on the closer neighbors, especially the nearest neighbor. In particular, when $\gamma=0$, the interpolated-NN reduces to the common $k$-NN, and when $\gamma=\infty$, interpolated-NN reduces to 1-NN. \cite{belkin2018overfitting} showed that given any fixed $\gamma\in\mathbb{R}^+$, the interpolated estimator using (\ref{eqn:model}) is minimax rate-optimal for mean squared error in regression, but only provided a suboptimal upper bound for the Regret of binary classification tasks.

	To evaluate the predictive performance of a NN algorithm, we adopt the conventional measures as follows: assume the random testing data $X$ and its corresponding response variable $Y$ follow the same distribution as the training data, then
	\begin{eqnarray*}
		&&\mbox{Regression:}\;\;\;\text{MSE}(k,n,\gamma)=\mathbb{E}( (\widehat{\eta}_{k,n,\gamma}(X)-\eta(X))^2 ),\\
		&&\mbox{Classification:}\;\;\;\text{Regret}(k,n,\gamma)=P\left( \widehat{g}_{k,n,\gamma}(X)\neq Y \right)-P\left(g(X)\neq Y\right),
	\end{eqnarray*}
	where $\mathbb{E}$ and $P$ are \blue{the} expectation and probability measure with respect to the joint distribution of training data and testing data.
	
	For regression, MSE is adopted to evaluate the predictive \blue{accuracy of mean responses}. For classification, the Regret measures the difference between the testing accuracies of the estimated classifier and the \blue{oracle} Bayes classifier, i.e., the excessive mis-classification rate. The Regret can be equivalently rewritten as 
	\begin{eqnarray}\label{eqn:weight}
		\text{Regret}(k,n,\gamma)=\mathbb{E}\left[ 2|\eta(X)-1/2|P(\widehat{g}_{k,n,\gamma}(X)\neq Y) \right],
	\end{eqnarray}
	that is, the Regret can be viewed as a weighted mis-classification rate, where less weight is assigned on the region closer to decision boundary $\{x:\eta(x)=1/2\}$.

	\section{Rate Optimality of Interpolated-NN}\label{sec:rate}
	In this section, the convergence rate results for interpolated-NN are provided. We show that both MSE (of regression task) and Regret (of classification task) converge at the optimal rate. Besides, the statistical instability (defined by \citep{sun2016stabilized}) of interpolated-NN is evaluated, revealing that interpolated-NN is as stable (in \blue{terms} of asymptotic rate) as $k$-NN.

	\subsection{Model Setup}
	The theorems in this section are developed under assumptions as those in \cite{CD14} and \cite{belkin2018overfitting}. We state these assumptions as follows:
	\begin{enumerate}
		\item[A.1]$X$ is a $d$-dimensional random variable on a compact set and satisfies the following  regularity condition \cite{audibert2007fast}: there exists positive $(c_0,r_0)$ such that for any $x$ in the support $\mathcal{X}$,
		\begin{equation*}
			\lambda(\mathcal{X}\cap B(x,r))\geq c_0\lambda(B(x,r)),
		\end{equation*}
		for any $0<r\leq r_0$, where $\lambda$ denotes the Lebesgue measure on $\mathbb{R}^d$. 
		\item[A.2] The density of $X$ is between $[m_x,M_x]$ for some constants $0<m_x\leq M_x<\infty$.
		\item[A.3] Smoothness condition: $|\eta(x)-\eta(y)|\leq A\|x-y\|^{\alpha}$ for some $\alpha>0$.
		\item[A.4] For classification, the model satisfies \blue{Tsybakov margin condition \cite{tsybakov2004optimal}}:  $P(|\eta(X)-1/2|<t)\leq Bt^{\beta}$
		\item[A.5] For regression, the variance of noise is finite, i.e., $\sigma^2(x)\leq M_{\sigma}<\infty$ for any $x\in\mathcal{X}$.
	\end{enumerate}
	
	\blue{Assumption A.1 regulates the shape of $\mathcal X$ and avoids spiky support, it essentially ensures that for any $x\in\mathcal X$, all its $k$ nearest neighbors are sufficiently close to $x$ with high probability.} If $\mathcal{X}$ is compact and convex, this regularity condition is automatically satisfied. 
	Assumption A.2 regularizes the neighborhood set for any testing sample $x\in\mathcal X$:
	the upper bound of the density prevents neighbors from clustering at $x$, i.e., $R_i$'s converge to 0 too fast; the lower bound of the density prevents the neighbors from being too far from $x$. Assumption A.1 and A.2 together guarantee that $R_i\asymp (k/n)^{1/d}$ in probability. If this lower bound assumption is violated and the density of $x$ goes to zero,
	classical $k$-NN may not perform well, and a locally-weighted NN \citep{cannings2017local} allows the weight assignment to depend on $x$, is preferred. As shown in \cite{cannings2017local}, locally-weighted NN improves the empirical performance of $k$-NN.  Readers of interest can design interpolated locally-weighted NN algorithm and study its asymptotic behavior as those in our later theorems. \blue{The constants $c_0, r_0, A, B, \beta, \alpha$ are distribution-specific (i.e., can change w.r.t. the distribution of $(X,Y)$ or the dimension $d$) but are independent to $n$. Among them, only $\alpha$ and $\beta$ are the two key values that affect the convergence rate (w.r.t. $n$) of the NN estimator.}

	For regression tasks, assumption A.1-A.3 and A.5 are commonly used in the literature (e.g., \citealp{belkin2018overfitting}), under which one can prove that the NN type regression estimators achieve the minimax rate $O(n^{-2\alpha/(2\alpha+d)})$, via the technical tools of \citep{tsybakov2004optimal}. For assumption A.3, a smaller value of $\alpha$, i.e., a less smooth true $\eta$, implies that estimating $\eta$ is more difficult. Assumption A.5 is imposed to restrict the variation of $y$ given $x$. Assumption A.5 only requires that the variance at each $x$ is finite, which is a \blue{fairly} weak condition.   
	
	For classification tasks, the smoothness condition A.3 and margin condition A.4 usually appear in the Regret analysis of $k$-NN \citep[e.g.,][]{CD14,belkin2018overfitting,sun2016stabilized}. These two conditions together describe how large the probability measure is near the decision boundary $\{x|\eta(x)=1/2\}$. Under A.1 to A.4, it is well known that the optimal rate for Regret of nonparametric classification is  $O(n^{-\alpha(\beta+1)/(2\alpha+d)})$ due to \cite{tsybakov2004optimal}. Note that our theorem allows that $\gamma=0$, i.e., it applies to $k$-NN as well.
	
	Besides the detailed analysis in the following sections, we also provide a concrete example with an intuitive explanation of why interpolation does not hurt the convergence of interpolated-NN in Section \ref{sec:knn_wnn_compare}.
	
	\subsection{Rate of Convergence}
	In the following theorem, we present the convergence rate of interpolated-NN for both regression and classification under a mild level of data interpolation, i.e., $\gamma$ is within some suitable range.
	\blue{This theorem is a refined result of \cite{belkin2018overfitting}, which only obtains the optimal convergence rate of MSE for regression tasks. Note that our rate for MSE is the same as in \cite{belkin2018overfitting}. We include both regression and classifications results in the following theorem for the sake of completeness. }
	
	\begin{theorem}\label{general}
		Assume $d-3\gamma\geq C>0$ for some constant $C>0$ and $\gamma\geq 0$. For regression, under A.1-A.3 and A.5,
		$$\text{MSE}(\gamma,n):=\min_{k}\text{MSE}(k,\gamma,n)= O(n^{-2\alpha/(2\alpha+d)}).$$
		For classification, under A.1-A.3 and A.4 if $\beta<2$,
		$$\text{Regret}(\gamma,n):=\min_k \text{Regret}(k,\gamma,n)= O(n^{-\alpha(\beta+1)/(2\alpha+d)}).$$
		In addition, denote $\kappa(\beta)$ as the smallest even number that is greater than $\beta+1$. If $\beta\geq 2$, when $d-\kappa(\beta)\gamma>0$, 
		$$\text{Regret}(\gamma,n)= O(n^{-\alpha(\beta+1)/(2\alpha+d)}).$$
	\end{theorem}

	In below we discuss the critical steps of proving Theorem \ref{general}. For regression, MSE at $x$ can be decomposed into
	\begin{align*}
		\mathbb{E}\big\{ (\widehat{\eta}_{k,n,\gamma}(x)-\eta(x))^2 \big\}\leq A^2\mathbb{E}\bigg[\sum_{i=1}^kW_i\|X^i-x\|^{\alpha}\bigg]^2+ \sum_{i=1}^k \mathbb{E}\bigg[W_i^2(Y(X^i)-\eta(X^i))^2 \bigg],
	\end{align*}
	and our convergence rate of MSE is thus based on careful derivations of upper bounds for $\mathbb{E}W_i$, $\mathbb{E}W_i^2$, and $\mathbb{E}\|X^i-x\|^{\alpha}$. Although interpolation introduces obvious biased prediction at the training data points by forcing $\widehat\eta(X_i)=Y_i$, after taking expectation over both training and testing data, such a bias effect is averaged out. For classification, we note that $(X^i,Y^i)$'s for $i=1,...,k$ are i.i.d. samples within the ball $B(x,r)$ conditional on $R^{k+1}=r$. Consequently, we can apply (non-uniform) Berry-Esseen Theorem to approximate $P(\widehat{g}_{k,n,\gamma}(x)\neq g(x))$ by normal probability, i.e.,
	\begin{equation}\label{eqn:illustration}
		P(\widehat{\eta}_{k,n,\gamma}(x)<1/2| g(x)=1)\approx \Phi\left( \frac{1/2-\mathbb{E}(\widehat{\eta}_{k,n,\gamma}(x))}{\sqrt{Var(\widehat{\eta}_{k,n,\gamma}(x))}}\right),
	\end{equation}
	which is related to the mean and variance of $\widehat{\eta}_{k,n,\gamma}(x)$. Similar to the regression case, the pointwise mis-classification rate is barely affected by data interpolation when $x\neq X_i$, thus the Regret of interpolated-NN has the same rate as the traditional $k$-NN.
	For different $\beta$ regions ($\beta<2$ and $\beta\geq 2$), different restrictions on $\gamma$ are imposed for technical simplicity. This restriction is used to control the reminder term that appears in the non-uniform Berry-Esseen Theorem. A detailed proof for the Regret convergence part of Theorem \ref{general} is postponed to the supplementary material. For the MSE convergence part, we refer readers to  \cite{belkin2018overfitting,belkin2018does}.
	
	\blue{
		Our technical contributions of Theorem \ref{general} (and Theorem \ref{thm:main} later) lie in the 
		adaptations of the existing analysis framework of \cite{CD14,samworth2012optimal}. 
		\begin{enumerate}
			\item[(a)] Compared to Theorem 4.5 in \cite{belkin2018overfitting} that is derived based on the Chebyshev’s inequality, the Berry-Esseen Theorem applied in Theorem \ref{general} leads to a tighter Regret bound.
			\item[(b)] By the definition of $W_i$, the weighted samples $W_iY_i$ are no longer independent with each other, so some transformations are made to enable the use of concentration inequalities in the approximation (\ref{eqn:illustration}). 
	\end{enumerate}}
	
	The insight of Theorem \ref{general} is that under proper \blue{a} interpolation scheme and reasonable overfitting level, data interpolation does not hurt the rate optimality of the NN algorithm for both regression and classification tasks.  
	
	
	\subsection{Statistical Stability}\label{sec:cis}
	In this section, we explore how interpolation affects the statistical stability of NN algorithms in classification.
	For a stable classification method, it is expected that with high probability, the classifier can yield the same prediction label when being trained by different data sets sampled from the same population. Therefore, \cite{sun2016stabilized} introduced a type of statistical stability, named as classification instability (CIS), to quantify how stable the classifier is. Denote $\mathcal{D}_1$ and $\mathcal{D}_2$ as two i.i.d. training sets of the same sample size $n$. Then CIS for interpolated-NN is defined as
	$$\text{CIS}_{k,n}(\gamma)=P_{\mathcal{D}_1,\mathcal{D}_2,X}\left( \widehat{g}_{k,n,\gamma}(X,\mathcal{D}_1)\neq \widehat{g}_{k,n,\gamma}(X,\mathcal{D}_2) \right),$$
	where $\widehat{g}_{k,n,\gamma}(x,\mathcal{D}_j)$ is the predicted label of $\widehat{g}_{k,n,\gamma}$ at $x$ when the training data set is $\mathcal{D}_j$ for $j=1,2$. Therefore, the CIS can be viewed as a counterpart of the variance measure of nonparametric regression estimator:
	$E_X Var(\widehat\eta_{k,n,\gamma}(X))=\frac{1}{2}E_{\mathcal{D}_1,\mathcal{D}_2,X}(\widehat{\eta}_{k,n,\gamma}(X,\mathcal{D}_1)- \widehat{\eta}_{k,n,\gamma}(X,\mathcal{D}_2))^2$.
	
	From the formula of CIS, a larger value of CIS indicates that the classifier is less statistically stable. Both misclassification rate and classification instability should be taken into account when evaluating the merits of any classification algorithm. Thus, our theory aims to characterize the CIS of interpolated-NN under the choice of $k$ that attains its optimal classification performance \blue{(i.e., the Regret)}.
	Besides, it is noteworthy that CIS is different from the algorithmic stability in the literature \citep{bousquet2002stability,hardt2016train,chen2018stability}, where the two data sets $\mathcal{D}_1$ and $\mathcal{D}_2$ are identical except for one sample, rather than independent identically distributed.

	The following theorem studies the CIS of interpolated-NN. In short, we show that the CIS for interpolated-NN and $k$-NN converge under the same rate:
	\begin{theorem}\label{thm:cis_rate}
		Under A.1, A.2, A.3 and A.4, if $\gamma$ satisfies the conditions stated in Theorem \ref{general}, when taking $k=cn^{1/(2\alpha+d)}$ for any constant $c>0$, we have
		\begin{equation*}
			\text{CIS}_{k,n}(\gamma)=O\left( n^{-\alpha\beta/(2\alpha+d)} \right).
		\end{equation*}
		Note that $k\asymp n^{1/(2\alpha+d)}$ is the choice of $k$, which attains the optimal rate of Regret for both $k$-NN and interpolated-NN.
	\end{theorem}
	Based on Theorem \ref{thm:cis_rate}, the rate of CIS of interpolated-NN is the same as traditional $k$-NN (i.e., taking $\gamma=0$) when we either (i) select the best $k$ value for $k$-NN and apply it to both $k$-NN and interpolated-NN, or (ii) select the optimal $k$ for $k$-NN and interpolated-NN respectively. Also, this rate of CIS matches the minimax lower bound described below.
	Let $P$ be the marginal distribution of $X$, and $P\otimes\eta$ denotes the joint distribution of $(X,y)$ determined by $P$ and conditional mean function $\eta$, then the following proposition holds:
	\begin{proposition}[Theorem 4 in \cite{sun2016stabilized}]
		Let $\mathcal{P}_{\alpha,\beta}$ be the set of probability distributions that contains all distributions of form $P\otimes\eta$, where $P$ and $\eta$ satisfy A.1, A.2, A.3 and A.4. If $\alpha\beta<d$, then there exists some constant $C>0$ such that for any estimator $\hat{\eta}$,
		\begin{eqnarray*}
			\sup\limits_{P\otimes\eta\in \mathcal{P}_{\alpha,\beta} } CIS(\hat{\eta}) \geq Cn^{-\alpha\beta/(2\alpha+d)}.
		\end{eqnarray*}
	\end{proposition}
	
	\section{Quantification of Interpolation Effect}\label{sec:const}
	Our results presented in Section \ref{sec:rate} show that interpolated-NN retains the rate minimaxity in terms of MSE, Regret, and statistical stability. To further reveal the subtle relationship between data interpolation and estimation performance, we sharply quantify the multiplicative \blue{constants} associated with the convergence rates of MSE, Regret, and statistical instability of interpolated-NN algorithm. 
	
	\subsection{Model Assumptions}
	To facilitate our theoretical investigation, a slightly different set of assumptions are imposed:
	\begin{enumerate}
		\item [A.1'] $X$ is a $d$-dimensional random variable on a compact $\mathbb{R}^d$ manifold $\mathcal{X}$ with boundary $\partial \mathcal{X}$.
		\item[A.2'] The density of $X$ is in $[m_x,M_x]$ for some $0<m_x\leq M_x<\infty$, and twice differentiable.
		\item [A.3'] For classification, the set $\mathcal{S}=\{ x|\eta(x)=1/2 \}$ is non-empty. There exists an open subset $U_0$ in $\mathbb{R}^d$ which contains $\mathcal{S}$ such that, for an open set containing $\mathcal{X}$ (defined as $U$), $\eta$ is continuous on $U\backslash U_0$.
		\item [A.4'] For classification, there exists some constant $c_x>0$ such that when $|\eta(x)-1/2|\leq c_x$, $\eta$ has bounded fourth-order derivative; when $\eta(x)=1/2$, the gradient $\dot{\eta}(x)\neq 0$ when $\eta(x)=1/2$, and with restriction on $x\in\partial\mathcal{X}$, $\dot{\partial\eta}(x)\neq 0$ if $\eta(x)=1/2$, where $\partial\eta(x)$ denotes the restriction of $\eta$ to $\partial \mathcal{X}$. Also the derivative of $\eta(x)$ within restriction on the boundary of support is non-zero.
		
		\item [A.5'] For regression, the second-order derivative of $\eta$ is smooth for all $x$.
		\item [A.6'] For regression, $\sup\limits_{x\in\mathcal{X}}\sigma^2(x)\leq M_{\sigma}<\infty$ and $\sigma^2(x)$ is twice-continuously differentiable.
	\end{enumerate}
	
	The smoothness condition A.4' describes how smooth the function $\eta$ is, which affects the minimax performance of $k$-NN given a class of $\eta$'s. \blue{$c_x$ is a distribution-specific value. The existence of nonzero $c_x$ and the high-order smoothness condition allows us to rigorously
		analyze the near decision boundary region where most of \blue{the} mispredictions are made (Step 2 in the proof of Theorem \ref{thm:main} in Appendix \ref{sec:classification_proof}). Note that the validity of our results doesn't depend on the \blue{magnitude} of $c_x$.} Assumptions A.3' and A.4' describe how far the samples are away from $\{x|\eta(x)=1/2 \}$. The assumptions are mostly derived from the framework established by \cite{samworth2012optimal}. Note that the additional smoothness required in $\eta$ and $f$ is needed to facilitate the asymptotic study of the interpolated weighting scheme. We also want to point out that these assumptions are generally stronger than A.1 to A.5, \blue{but they allow us to obtain a sharper bound. Using A.1 to A.5, one can only obtain some upper bound rate for the approximation error in (\ref{eqn:illustration}), while using A.1' to A.6' one can obtain the exact value of the approximation error in (\ref{eqn:illustration}) (except for high order remainder terms)}.
	
	\begin{remark}\label{rem:con}
		There is a heuristic relationship between smoothness \blue{conditions} A.1-A.4 and A.1'-A.4'. 
		First of all, Assumption A.1' to A.4' also imply the marginal condition A.4 with $\beta=1$. In addition, as mentioned by \cite{CD14}, A.2 in fact implies 
		\begin{eqnarray}\label{eqn:connect}
			|\eta(x)-\eta(B(x,r))|\leq Lr^{\alpha},
		\end{eqnarray}
		while \cite{samworth2012optimal} showed the approximation  $\mathbb{E}\eta(X)\approx\eta(x_0)+tr[\ddot{\eta}(x_0)\mathbb{E}(X-x_0)(X-x_0)^{\top}]$ under A.2' and A.4'. Therefore, by matching result (\ref{eqn:connect}) and the above approximation, we can view  A.1'-A.4' as a special case of conditions A.1-A.4 under smoothness parameter $\alpha=2$ and $\beta=1$. Furthermore, for classification, the minimax rate under conditions A.1-A.4 is $O(n^{-\alpha(\beta+1)/(2\alpha+d)})$, which becomes $O(n^{-4/(4+d)})$ under $(\alpha,\beta)=(2,1)$. It coincides with convergence rates result \citep{samworth2012optimal} under condition set A.1'-A.4'.
	\end{remark}			 
	
	\subsection{Main Theorem}\label{sec:main}
	
	The following theorem examines the asymptotics of MSE and Regret of interpolated-NN and $k$-NN under Assumptions A.1'-A.6'. We first define some useful quantities. Let $P_1$ and $P_2$ (and $f_1$, $f_2$) be the conditional distributions (and densities) of $X$ given $Y=0,1$ respectively, and $\pi_1,\pi_2$ be the marginal probability $P(Y=0)$ and $P(Y=1)$. Denote $P=\pi_1P_2+\pi_2P_2$,  $\bar{P}=\pi_1P_1-\pi_2P_2$,  $f(x)=\pi_1 f_1(x)+\pi_2 f_2(x)$, and also denote $\Psi(x)=\pi_1 f_1(x)-\pi_2 f_2(x)$. Define \begin{eqnarray*}\label{eqn:a}
		a(x)&=&\frac{1}{f(x)^{1+2/d}d}\left\{ \sum_{j=1}^d [\dot{\eta}_j(x)\dot{f}_j(x)+\ddot{\eta}_{j,j}(x)f(x)/2 ] \right\}.
	\end{eqnarray*}
	Then the following decomposition of MSE/Regret holds: 
	\begin{theorem}\label{thm:main}
		Assume $d-3\gamma\geq C>0$ for some constant $C$ and $\gamma\geq 0$. For regression, suppose that assumptions A.1', A.2', A.5', and A.6' hold. If $k$ satisfies $n^{\delta}\leq k \leq n^{1-4\delta/d}$ for some $\delta>0$, then
		\begin{eqnarray}
			\text{MSE}(k,n,\gamma) &=&  \underbrace{\frac{(d-\gamma)^2}{d(d-2\gamma)}\frac{1}{k}\mathbb{E}[\sigma^2(X)]\label{eqn:variance2}}_{Variance}\\&&+ \underbrace{\frac{(d-\gamma)^2}{(d+2-\gamma)^2}\frac{(d+2)^2}{d^2}\mathbb{E}\left(a^2(X)\mathbb{E}^2(R_1^2|X)\right)\label{eqn:bias2}}_{Bias}+Remainder,
		\end{eqnarray}
		where $Remainder=o(\text{MSE}(k,n,\gamma))$.
		
		For classification, under A.1' to A.4', the excess risk w.r.t. $\gamma$ becomes 
		\begin{eqnarray}
			\text{Regret}(k,n,\gamma)&=&\underbrace{ \frac{(d-\gamma)^2}{d(d-2\gamma)}\frac{1}{4k}\int_{S}\frac{f(x_0)}{\|\dot{\eta}(x_0)\|}d\text{Vol}^{d-1}(x_0)\label{eqn:variance}}_{Variance}\\&&+\underbrace{\frac{(d-\gamma)^2}{(d+2-\gamma)^2}\frac{(d+2)^2}{d^2}\int_{S}\frac{f(x_0)a(x_0)^2}{\|\dot{\eta}(x_0)\|} \mathbb{E}(R_1^2|X=x_0)\label{eqn:bias} d\text{Vol}^{d-1}(x_0)}_{Bias}\\&&\nonumber+Remainder,
		\end{eqnarray}
		where $Remainder=o(\text{Regret}(k,n,\gamma))$.
		
	\end{theorem}
	
	The proof of Theorem \ref{thm:main} is postponed to Section \ref{sec:cor:proof} of the supplementary material. The basic ideas are the same as Theorem \ref{general}. Since the stronger conditions of A.1' to A.6' enables a detailed Taylor expansion on MSE/Regret, the multiplicative constants associated with variance and bias can be figured out.

	Similar as in \cite{samworth2012optimal}, when taking $k=n^{4/(d+4)}$, $\mathbb{E}(R_1^2|X)=O(n^{-2/(d+4)})$, thus it is easy to see that the MSE and Regret of interpolated-NN are of $O(n^{-4/(d+4)})$, which is the same rate of $k$-NN, optimal weighted NN \citep{samworth2012optimal}, stabilized-NN \citep{sun2016stabilized}, and distributed-NN \citep{duan2018DNN}.
	
	Theorem \ref{thm:main} provides an exact variance-bias decomposition for MSE and Regret, except for a negligible remainder term, which enables us to quantify the effect of interpolation carefully.
	Given fixed $(k,n)$ satisfying the conditions in Theorem \ref{thm:main}, the coefficient in the variance terms (\ref{eqn:variance2}) and (\ref{eqn:variance}) are increasing functions in $\gamma$, i.e.,
	\begin{eqnarray*}
		\frac{1}{k} \frac{(d-\gamma)^2}{d(d-2\gamma)}=\frac{1}{k}\left(1+\frac{\gamma^2}{d(d-2\gamma)}\right)\approx \frac{1}{k}\left(1+\frac{\gamma^2}{d^2}\right),
	\end{eqnarray*} which behaves like a quadratic function around $\gamma=0$; the coefficient in the bias terms in (\ref{eqn:bias2}) and (\ref{eqn:bias}) are decreasing functions in $\gamma$, i.e.,
	\begin{eqnarray*}
		\frac{(d-\gamma)^2}{(d+2-\gamma)^2}\frac{(d+2)^2}{d^2}=\left(1-\frac{2\gamma}{d^2+2d-\gamma d}\right)^2\approx \left(1-\frac{4\gamma}{d^2+2d} \right)
	\end{eqnarray*} which behaves like a linear decreasing function around $\gamma=0$. Therefore, it is clear that tuning the interpolation level $\gamma$ leads to a trade-off between the bias and variance of interpolated-NN estimation: interpolated-NN will introduce a larger variance and reduce the bias when increasing $\gamma$. In addition, given a small value of $\gamma$, since a quadratic increase of variance ($O(\gamma^2)$) is always dominated by a linear decrease of bias ($O(\gamma)$), it is possible for interpolated-NN to achieve a better performance than $k$-NN.

	The above analysis can be extended to a general class of weight functions, and the same variance-bias trade-off occurs, in the sense that a weighting scheme allocating more weights on closer neighbors tends to have a smaller bias and larger variance. The detailed analysis can be found in Section \ref{sec:ownn} of the supplementary material.

	\subsection{U-shaped Asymptotic Performance of Interpolated-NN}\label{sec:two}
	
	In this section, we aim to refine our results stated in Theorem \ref{thm:main}, and characterize the asymptotic performance of interpolated-NN with respect to the interpolation level $\gamma$. Two choices of $k$ are considered. 
	In first choice, $k$ is $\gamma$-independent and only relies on the value of $n$. To be more specific, we choose $k = \argmin \mbox{MSE}(k,n,\gamma=0)$ (or $k = \argmin \mbox{Regret}(k,n,\gamma=0)$ for classification) which is the optimal choose for $k$-NN algorithm. This same $k$ value is used regardless of the interpolation level, and we study how the performance of interpolation-NN changes as $\gamma$ increases. 
	In the second choice, we allow $k$ to depend on interpolation level $\gamma$ as well (hence is denoted by $k_\gamma$), such that $k_\gamma = \argmin \mbox{MSE}(k,n,\gamma)$ (or $k_\gamma = \argmin \mbox{Regret}(k,n,\gamma)$ for classification). In other words, the $k$ value is optimally tuned with respect to the interpolation level.

	Under the first choice of $k$, we quantify the asymptotic performance of interpolated-NN using $k$-NN as the benchmark reference. Define
	$$\text{PR}(d,\gamma):=\frac{1}{1+d/4}\left(\frac{(d-\gamma)^2}{d(d-2\gamma)} +\frac{d}{4}\frac{(d-\gamma)^2(d+2)^2}{d^2(d+2-\gamma)^2}\right).$$
	Then the following result holds:
	\begin{corollary}\label{coro:same_k}
		Under conditions in Theorem \ref{thm:main}, for regression, when $k$ is chosen to minimize the MSE of $k$-NN, then \blue{for any $\gamma\in[0,d/3)$, }the performance ratio satisfies, 
		\begin{equation*}
			\frac{\text{MSE}(k,n,\gamma)}{\text{MSE}(k,n,0)}\rightarrow\text{PR}(d,\gamma).
		\end{equation*} 
		For classification, when $k$ is chosen to minimize the Regret of $k$-NN, then \blue{for any $\gamma\in[0,d/3)$, }the performance ratio becomes \begin{eqnarray*}\label{eqn:same_k}
			\frac{\text{Regret}(k,n,\gamma)}{\text{Regret}(k,n,0)}\rightarrow \text{PR}(d,\gamma).
		\end{eqnarray*}
		
	\end{corollary}

	When $k$ is chosen based on $k$-NN, the interpolation affects regression and classification in exactly the same manner through $\text{PR}(d, \gamma)$. In particular, this ratio exhibits an interesting U-shape curve of $\gamma$ for any fixed $d$: as $\gamma$ increases from $0$,  $\text{PR}(d,\gamma)$ first decreases and then increases. Therefore, within the range $(0, \gamma_d)$ for some  threshold value $\gamma_d$ which depends on dimension $d$ only, the asymptotic performance ratio $\text{PR}(d,\gamma)<1$. It is noteworthy that although $\text{PR}(d,\gamma)$, as a function of $\gamma$, is U-shaped on $\gamma\in\mathbb{R}$, our Corollary \ref{coro:same_k} is restricted to that $\gamma<d/3$. Thus $\gamma_d$ is the minimum value between $\gamma/3$ and the second root of function $\text{PR}(d,\gamma)=1$. When $\text{PR}(d,\gamma)<1$, the interpolated-NN is {\em strictly} better than the $k$-NN. Moreover, the effect of interpolation is asymptotically invariant for any distributions of $(X,y)$ as long as the assumptions hold, and $\text{PR}(d,\gamma)\rightarrow 1$ as the dimension $d$ grows to infinity.

	In the second choice, we let the interpolated-NN utilize the optimal $k$ with respect to $\gamma$.  From Theorem \ref{thm:main}, $k_{\gamma}\asymp k_0(:=k_{\gamma=0})\asymp n^{\frac{4}{d+4}}$, but $k_{\gamma}/k_0>1$ for $\gamma>0$, i.e., interpolated-NN needs to employ slightly more neighbors than $k$-NN to achieve the best performance. 
	
	Corollay~\ref{coro} below asymptotically compares interpolated-NN and traditional $k$-NN in terms of the above measures. It turns out that the performance ratio (the ratio of two MSE's or two Regret's), asymptotic converges to 
	$$\text{PR}'(d,\gamma):=\left( 1+\frac{\gamma^2}{d(d-2\gamma)} \right)^{\frac{4}{d+4}}\left( \frac{(d-\gamma)^2}{(d+2-\gamma)^2}\frac{(d+2)^2}{d^2} \right)^{\frac{d}{d+4}},$$
	which is a function of $d$ and $\gamma$ only, and is independent of the underlying data distribution.
	
	\begin{corollary}
		\label{coro}
		Under conditions in Theorem \ref{thm:main}, for any $\gamma\in[0,d/3)$, 
		\begin{eqnarray*}
			\frac{\text{MSE}(n,\gamma)}{\text{MSE}(n,0)}\rightarrow \text{PR}'(d,\gamma),\quad \mbox{ and  }\quad
			\frac{ \text{Regret}(n,\gamma) }{\text{Regret}(n,0)}\rightarrow  \text{PR}'(d,\gamma), \quad\mbox{ as }n\rightarrow \infty
			.		\end{eqnarray*}
		
		Note that $\text{MSE}(n,0)$/$\text{Regret}(n,0)$ is the optimum \text{MSE}/\text{Regret} for $k$-NN.
	\end{corollary}
	
	Denote $\gamma_d'$ as the threshold such that $\text{PR}'(d,\gamma)<1$ for any $\gamma<\gamma_d'$. From Corollary \ref{coro}, starting from 1 when $\gamma=0$, the performance ratio $\text{PR}'(d,\gamma)$ will first decrease in $\gamma$ then increase; see Figure \ref{fig:my_label} in introduction.  Some further calculations can show that $\gamma_d'<d/3$ when $d\leq 3$; $\gamma_d'=d/3$ when $d\geq4$. It can also be figured out that whether interpolated-NN is better or not does not depends on the distribution of $(X,y)$.
	
	In addition, comparing the results in Corollary \ref{coro:same_k} and \ref{coro}, since interpolated-NN selects the $k$ to minimize its Regret in $\text{PR}'(d,\gamma)$, the region of $\gamma$ where $\text{PR}'(d,\gamma)$ is smaller than 1 is wider than that for $\text{PR}(d,\gamma)$ as in Corollary \ref{coro:same_k}, i.e.,$\gamma_d'>\gamma_d$. 
	
	\begin{remark}\label{rem:lard}
		It is easy to show that \blue{ the limiting values of $PR$ and $PR'$ converge to 1 in $d$, i.e., $\lim_{d\to\infty}[\min_{\gamma<d/3}\text{PR}(d,\gamma)]=1 $ and $\lim_{d\to\infty}[\min_{\gamma<d/3}\text{PR}'(d,\gamma)]=1$}. This indicates that high dimensional model benefits less from interpolation, or said differently, high dimensional model is less affected by data interpolation. This phenomenon can be explained by the fact that, as $d$ increases, $R_i/R_{k+1}\rightarrow 1$ for $i=1,..,k$ due to high dimensional geometry.  
	\end{remark}
	
	On the other hand, \cite{samworth2012optimal} worked out a general form of Regret and MSE for a weighted NN estimator and thereafter proposed the optimally weighted nearest neighbors algorithm (OWNN) by minimizing \text{Regret} for classification with respect to the values of weight $W_i$'s. We extend their results to regression and compare them with interpolated-NN.
	
	Combining Theorem  \ref{coro} with \cite{samworth2012optimal}, we have
	$$ \frac{R(n,\text{OWNN})}{R(n,\gamma)}\rightarrow 2^{\frac{4}{d+4}}\left(\frac{d+2}{d+4}\right)^{\frac{2d+4}{d+4}}\left( 1+\frac{\gamma^2}{d(d-2\gamma)} \right)^{-\frac{4}{d+4}}\left( \frac{(d-\gamma)^2}{(d+2-\gamma)^2}\frac{(d+2)^2}{d^2} \right)^{-\frac{d}{d+4}}, $$
	which is always smaller than 1. Here $R(n,\text{OWNN})$ denotes the \text{MSE}/\text{Regret} of OWNN given its optimum $k$, and $R(n,\gamma)$ denotes the one of interpolated-NN given its own optimum $k$ choice. 
	Again, the above ratio reflects the difference in convergence for OWNN and interpolated-NN, at the multiplicative constant level. This ratio converges to 1 as $d$ diverges (just as the case of $\mbox{PR}(d,\gamma)$; see Remark~\ref{rem:lard}). Thus, under ultra high dimensional setting, the performance differences among $k$-NN, interpolated-NN, and OWNN are almost negligible even at the multiplicative constant level \blue{when $n\rightarrow\infty$}.

	\subsection{Statistical Stability}
	Similar to MSE and Regret, we perform a precise quantification analysis for the convergence of CIS: Theorem \ref{CIS} below illustrates how CIS is affected by interpolation.  In short,  if interpolated-NN and $k$-NN algorithms share the same value of $k$, then interpolated-NN is less stable than $k$-NN; otherwise, the interpolated-NN will be more stable for $\gamma\in(0,\gamma_d')$ if $k$ is allowed to be chosen optimally based on $\gamma$.
	
	The following theorem quantifies CIS for general choices of $(k,\gamma)$:
	\begin{theorem}\label{CIS}
		Under the conditions in Theorem \ref{thm:main}, the CIS of interpolated-NN is derived as
		$$\text{CIS}_{k,n}(\gamma)=(1+o(1))\frac{B_1}{\sqrt{\pi}}\frac{1}{\sqrt{k}} \sqrt{ \frac{(d-\gamma)^2}{d(d-2\gamma)} }.$$
	\end{theorem}
	The proof of Theorem \ref{CIS} is postponed to Section \ref{sec:cis:proof} in the supplementary material. Compared with Theorem \ref{thm:main}, one can figure out that CIS is related to the variance term in total Regret, but not related to the bias term. This is because the two data sets $\mathcal{D}_1$ and $\mathcal{D}_2$ are drawn independently from the same distribution, thus the estimators share the same bias, while how different the predictions \blue{depends} on the variance.
	
	In Section \ref{sec:cis}, we show that under either (i) using the optimum $k$'s for these two methods respectively; or (ii) using the same $k$ (the optimum for $k$-NN) for both $k$-NN and interpolated-NN, CIS for both $k$-NN and interpolated-NN converges in the same rate. Based on Theorem \ref{CIS}, we can further compare the multiplicative constants as in Corollary \ref{coro:cis} below:
	\begin{corollary}\label{coro:cis} 
		Following the conditions in Theorem \ref{CIS},  when the same $k$ value is used for $k$-NN and interpolated-NN (as long as the $k$ satisfies conditions in Theorem \ref{thm:main}), then as $n\rightarrow\infty$, $$\frac{\text{CIS}_{k,n}(\gamma)}{\text{CIS}_{k,n}(0)}\rightarrow \sqrt{\frac{(d-\gamma)^2}{d(d-2\gamma)}}\geq 1.$$
		
		On the other hand, if we choose optimum $k$'s for $k$-NN and interpolated-NN respectively, i.e., $k_\gamma=\argmin_k \text{Regret}(k,n,\gamma)$, when $n\rightarrow\infty$, we have
		$$\left(\frac{\text{CIS}_{k_\gamma,n}(\gamma)}{\text{CIS}_{k_0,n}(0)}\right)^2\rightarrow \text{PR}'(d,\gamma).
		$$
		Therefore, when $\gamma\in(0,\gamma_d')$, interpolated-NN with optimal $k$ has higher accuracy and stability than $k$-NN at the same time.
		
	\end{corollary}
	
	From Corollary \ref{coro:cis}, the interpolated-NN is not as stable as $k$-NN if both algorithms use the same number of neighbors. However, this is not the case if an optimal $k$ is tuned w.r.t $\gamma$. An intuitive explanation is that under the same $k$, $k$-NN has a smaller variance (more stable) given equal weights for all $k$ neighbors. On the other hand, by choosing an optimum $k$, interpolated-NN can achieve a much smaller bias, which offsets its performance lost in the variance through enlarging $k$.

	\subsection{Effect of Corrupted Testing Data}
	When applying machine learning algorithms in practice for classification, the testing data may be corrupted, and its distribution differs from the training data. For example, the testing data may be randomly perturbed due to the inaccuracy of data collection (e.g., an image has noise due to an optical sensor system malfunction). In some other scenarios, the testing data may be deliberately modified to induce the learning algorithm to make a wrong decision (e.g., spam email sender will try to hack the email filter system). From the definition of interpolated-NN, the prediction is forced to jump to a label if the testing sample is sufficiently closed to some training data, and the estimated regression function $\eta(\cdot)$ could be rather rugged (refer to Figure \ref{toy}). These observations potentially imply a certain degree of adversarial instability in the sense that a small change of $x$ may lead to a huge change of $\widehat\eta_{k,n,\gamma}(x)$. This motivates us to investigate the performance of interpolated-NN further when encountering random perturbed/adversarial attacked input testing samples. In short, when the corruption is independent with training data, \blue{i.e.,} random perturbation or black-box attack, a small positive $\gamma$ improves the performance. When the corruption is data-dependent, \blue{i.e.,} white-box attack, interpolated-NN is vulnerable.  
	Some existing literature focused on design attack / adversarial robust nearest neighbors-type algorithms, e.g., \cite{wang2017analyzing,sitawarin2020minimum,xing2021predictive}, and they worked on un-interpolated NN algorithms.
	
	Denote $\omega$ as the corruption level of testing data, i.e., instead of observing the testing data $x$, we observe another value $\widetilde{x}$ which is inside an $\mathcal{L}_2$ ball $B(x,\omega)$ centering at $x$ with radius $\omega$. Three types of corruptions are considered in this paper: (1) random perturbation: $\widetilde{x}_{rand}$ is randomly drawn from the ball $B(x,\omega)$; (2) black-box attack: the adversary has no information on our training data and the trained model, hence it trains a different machine $\widetilde{\eta}(\cdot)$ using an independent dataset. The adversarial attack thereafter is designed as $\widetilde{x}_{black}=\argmax_{z\in B(x,\omega)} \widetilde{\eta}(z)$ if $\eta(x)<1/2$ and $\widetilde{x}_{black}=\argmin_{z\in B(x,\omega)} \widetilde{\eta}(z)$ if $\eta(x)>1/2$; (3) white-box attack: the adversary has full access to the trained model $\widehat \eta_{n,k,\gamma}$, and designed white-box attack as $\widetilde{x}_{white}=\argmax_{z\in B(x,\omega)} \widehat{\eta}_{n,k,\gamma}(z)$ if $\eta(x)<1/2$ and $\widetilde{x}_{white}=\argmin_{z\in B(x,\omega)} \widehat{\eta}_{n,k,\gamma}(z)$ if $\eta(x)>1/2$.
	
	We first display the sufficient condition when the testing Regret still converges in the rate of $n^{-4/(d+4)}$. 
	\begin{corollary}[Informal description]\label{coro:corruption}
		Under conditions in Theorem \ref{thm:main}, taking $k\asymp n^{4/(d+4)}$, for corruption scheme $\widetilde{x}_{rand}$ and $\widetilde{x}_{black}$, when $\omega=O(n^{-2/(d+4)})$, the Regret $P(\widehat{g}_{k,n,\gamma}(\widetilde{X})\neq Y)-P(g(X)\neq Y)=O(n^{-4/(d+4)})$. When $\omega=o(n^{-2/(d+4)})$, for corruption scheme $\widetilde{x}_{rand}$ and $\widetilde{x}_{black}$, the Regret ratio (Regret of interpolated-NN over Regret of $k$-NN) is the same as the one for un-corrupted data. When $\omega^3=o(n^{-4/(d+3)})$, the regret decreases when $\gamma$ slightly increases from zero.
		
		For corruption scheme $\widetilde{x}_{white}$, when $\omega=O(n^{-2/(d+4)}) \wedge o(n^{-1/d})$, the Regret $P(\widehat{g}_{k,n,\gamma}(\widetilde{X})\neq Y)-P(g(X)\neq Y)=O(n^{-4/(d+4)})$ for any $\gamma$ satisfying conditions in Theorem \ref{thm:main}. When $\omega=o(n^{-2/(d+4)})\wedge o(n^{-1/d})$, the Regret ratio (Regret of interpolated-NN over Regret of $k$-NN) is the same as the one for un-corrupted data. 
	\end{corollary}
	
	For the formal representations for random perturbation and black-box attack, we postpone them to Section \ref{sec:corruption} of the supplementary material.
	
	Corollary \ref{coro:corruption} reveals that when the corruption level is small enough, we can still benefit from interpolation since the performance ratio between interpolated-NN and $k$-NN is unchanged. However, interpolated-NN is more vulnerable to the white-box attack. 
	
	On the opposite, when $\omega$ is not sufficiently small, the testing data corruptions lead to a sub-optimal convergence rate, especially under white-box \blue{attacks}. The reason is that unless $\eta\in\{0,1\}$, there are always (infinitely many as $n$ increases) training samples, denoted as $\mathcal{D}'$, whose labels are different from the Bayes classifier. Then for any testing sample $x$, such that $B(x,\omega)$ overlaps with $\mathcal{D}'$, the white-box attack can be designed as $\widetilde x_{white} \in B(x,\omega)\cap \mathcal{D}'$, and will yield a wrong prediction, as the prediction of interpolated-NN always interpolate $\mathcal{D}'$.
	
	For $k$-NN, such a problem will not happen since the predictor always takes an average among $k$ training samples. Based on \cite{belkin2018overfitting}, the misclassified labels are dense in interpolated-NN. However, for (1) random perturbation and (2) black-box attack, the phenomenon is different: 
	\begin{corollary}
		Under the conditions of Theorem \ref{thm:main}, for white-box attack $\widetilde{x}_{white}$, when  $\omega/(n^{-1/d})\rightarrow\infty$, there exists some constant $c>0$ (depending on $f$ and $\eta$) such that the Regret is asymptotically greater than $c$. For random perturbation $\widetilde{x}_{rand}$ and black-box attack $\widetilde{x}_{black}$,  when  $\omega/(n^{-1/d})\rightarrow\infty$, the Regret is in $O(\omega^2)$.
	\end{corollary}

	\subsection{Bless of Interpolation and Effect of Distance Metric}\label{sec:appendix:metric}
	As shown by the results in \blue{the} previous section, data interpolation leads to a more accurate and stable performance than traditional $k$-NN when $k=k_\gamma$ and  $\gamma\in(0,\gamma_d')$. Similarly, in the literature of double descent, \citep[e.g.,][]{belkin2019two,hastie2019surprises} use a linear regression model to demonstrate a U-shaped curve (w.r.t. data dimension $d$) of MSE in the over-fitting regime (i.e., the number of linear covariates is larger than the number of observations). 
	
	To compare interpolated-NN and linear regression, although we observe a U-shaped performance ratio curve for both interpolated-NN estimator and ridgeless estimator in high dimensional linear regression when over-fitting occurs, they benefit from interpolation in different ways. For the ridgeless estimator in the regression problem, an increasing over-parametrization level causes a larger bias, but the coefficients all tend to zero. Therefore, following the analysis in \cite{belkin2019two} and \cite{hastie2019surprises}, the variance is very small and leads to a descent of MSE in the over-fitting regime. This descent in MSE in this interpolation regime is the second descent of MSE in the double descent phenomenon.
	
	On the other hand, for interpolated-NN, the benefit of increasing the level of interpolation is the reduction of bias. As a result, although the double descent phenomenon is observed in many different estimation procedures \citep[e.g.,][]{belkin2018reconciling}, a detailed study is needed to comprehensively understand how each machine learning technique enjoys the benefit of interpolation.

	Besides, it is worth mentioning that the benefit of interpolation in interpolated-NN is not affected by \blue{the} choice of distance metric. In the previous discussions, the $\mathcal{L}_2$ norm is used to determine the neighborhood set and measure the distance $R_i$. If we replace it by a different norm, e.g., $\mathcal{L}_1$ or $\mathcal{L}_{\inf}$, it turns out that the effect of interpolation is the same as under $\mathcal{L}_2$ measure:

	\begin{corollary}\label{coro:metric}
		Assume the optimum $k$'s for $k$-NN and interpolated-NN are chosen, respectively. Under assumptions in Corollary \ref{coro}, if $\mathcal{L}_p$ ($p\geq 1$) distance is used to calculate the distances among all data points, as well as for deciding interpolation weighting scheme, then the performance ratio between interpolated-NN and $k$-NN shares the same shape as in Corollary \ref{coro}, and CIS ratio shares the same shape as in Corollary \ref{CIS}.
	\end{corollary}
	
	\section{Numerical Experiments}\label{sec:numerical}
	In this section, several simulation studies are presented to justify our theoretical discoveries for regression, classification, and stability of the interpolated-NN algorithm, together with some real data analysis.

	\subsection{Simulations}
	In Section \ref{sec:two}, two scenarios of the performance of interpolated-NN are presented: (1) interpolated-NN utilizes the same $k$ as $k$-NN; (2) $k$ is chosen optimally for each $\gamma$. In Section \ref{sec:numerical:same}, the experiment setups are described in details and the numerical results are presented for scenario (1). The numerical results for scenario (2) are postponed to Section \ref{sec:numerical:diff}.
	
	\subsubsection{$k$ is Chosen Optimally for $\gamma=0$}\label{sec:numerical:same}
	We aim to estimate the performance ratio curve via numerical simulations and compare it with the theoretical curve 
	$\text{PR}(d,\gamma)$ (in Corollary \ref{coro:same_k}). 
	We take training sample size $n=2048$ and use 5000 testing samples to evaluate MSE/Regret/CIS for all simulations. This procedure is repeated 500 times to obtain the mean and standard deviation of the performance ratios. As observed in \cite{samworth2012optimal}, the empirical performance ratio is not stable, with $\pm 0.02$ difference to the theoretical value when $n=1000$. Since our goal of this simulation is to empirically validate the performance ratios rather than promoting an interpolated-NN based method, instead of tuning $k$ via cross-validation over training data, we select $k$ which minimizes the MSE/Regret over independently simulated large testing data set.

	For regression simulation, each attribute of $X$ follows i.i.d. uniform distribution on $[-1,1]$ with $d=2$ or $5$, and $Y=\eta(x)+\varepsilon$ with $\varepsilon\sim N(0,1)$ where
	\begin{eqnarray}\label{eqn:simualtion_model}
		\eta(x)&=&\frac{e^{x^{\top}w}}{e^{x^{\top}w}+e^{-x^{\top}w}},\\\qquad
		w_i&=& i-d/2-0.5\qquad i=1,...,d.\nonumber
	\end{eqnarray} 
	A sequence of levels of interpolation $\gamma/d=0,0.05,0.1,\ldots,0.35$ are used to evaluate the performance of interpolated-NN.
	\begin{figure}[!ht]
		\centering
		\includegraphics[trim=0 20 50 0,clip,scale=0.65]{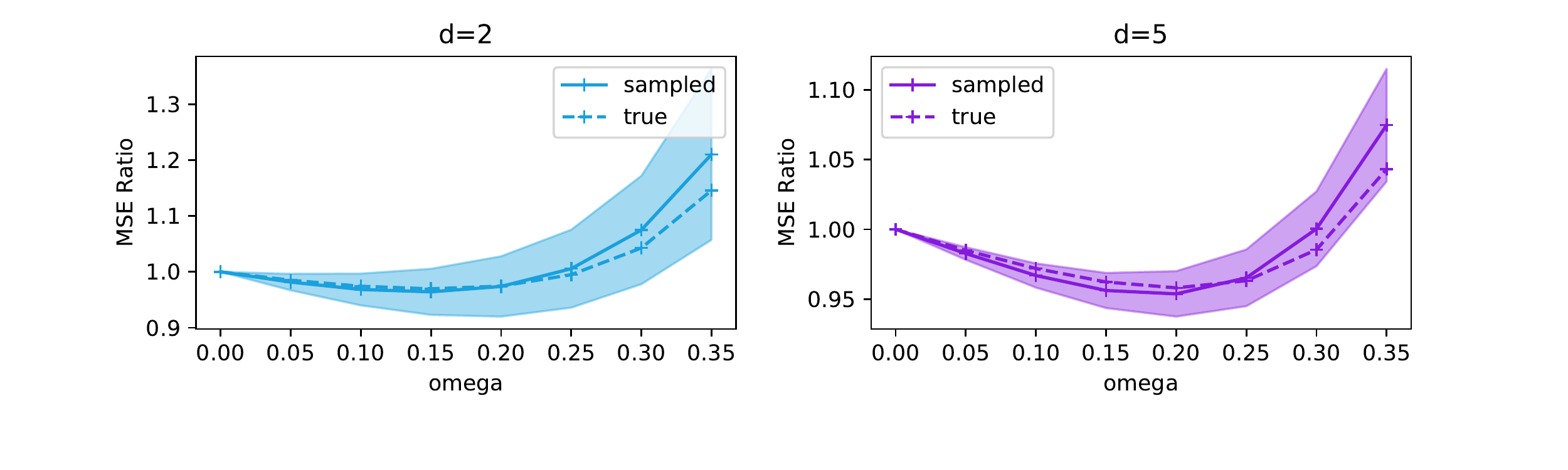}
		\caption{Performance Ratio in Regression when $k$ is Chosen Optimally for $k$-NN}
		\label{fig:regre3_same_k}
	\end{figure}

	The results are summarized in Figure \ref{fig:regre3_same_k}.  The trends for theoretical value and simulation value are close. The small difference is mostly caused by the small order terms in the asymptotic result and shall vanish if larger $n$ is used. Note that $\gamma/d=0.35$ is outside our theoretical range $\gamma/d<1/3$, but the empirical performance is still reasonable. 
	
	For classification, two models are generated. For the first model, each dimension of $P_1$, the marginal distribution of $X$ given $Y=0$, follows independent standard Cauchy. The first $\lfloor d/2\rfloor$ dimensions of $P_2$, the marginal distribution of $X$ given $Y=1$, follows independent standard Cauchy, and the remaining dimensions of $P_2$ \blue{follow} independent standard Laplace distribution. Through the design of the first model, the two classes have the same center at the origin. In the second model, each dimension of $P_1$ follows independent $N(0,1)/2+N(3,4)/2$, and each dimension of $P_2$ follows independent $N(1.5,1)/2+N(4.5,4)/2$. For both $P_1$ and $P_2$ in the second model, the distributions are multi-modal.

	The CIS of interpolated-NN classifier is estimated by calculating the proportion of testing samples that have different prediction labels, that is
	\begin{eqnarray*}
		\widehat{\text{CIS}}(\gamma)=\frac{1}{n}\sum_{i=1}^n1( \widehat{g}_{k,n,\gamma}(x_i,\mathcal{D}_1)\neq \widehat{g}_{k,n,\gamma}(x_i,\mathcal{D}_2) ),
	\end{eqnarray*}
	where $D_1$, $D_2$ are two independent training data sets, and $x_i$'s are independently sampled testing data points.

	Similar to \blue{Figure \ref{fig:regre3_same_k}}, the classification results are summarized in \blue{Figures} \ref{fig:model3_same_k} and \ref{fig:model2_same_k}.
	\begin{figure}[!ht]
		\centering
		\includegraphics[trim=0 0 50 0,clip,scale=0.65]{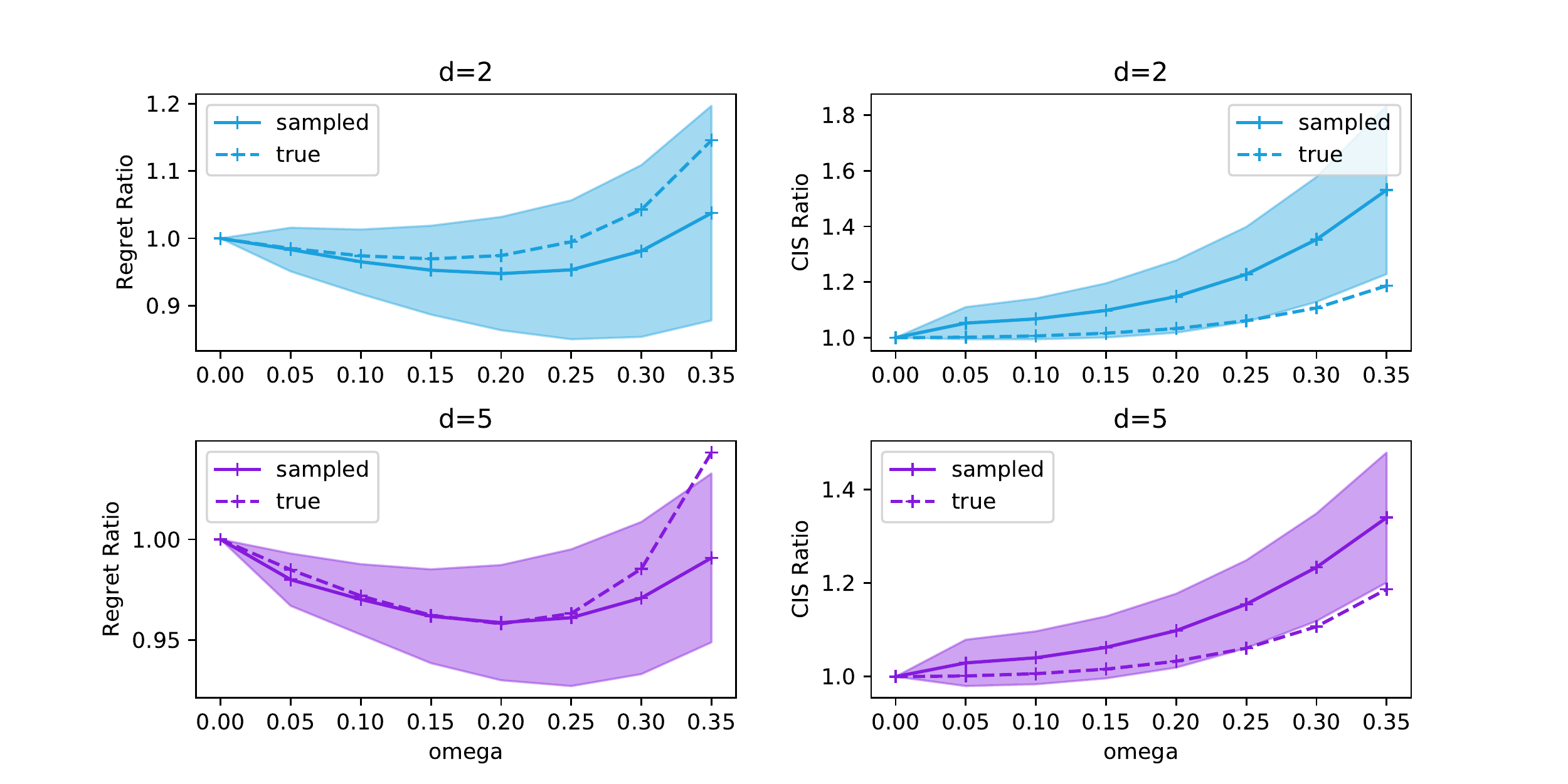}
		\caption{Performance Ratio in Classification, Model 1, when $k$ is Chosen Optimally  for $k$-NN}
		\label{fig:model3_same_k}
	\end{figure}
	\begin{figure}[!ht]
		\centering
		\includegraphics[trim=0 0 50 0,clip,scale=0.65]{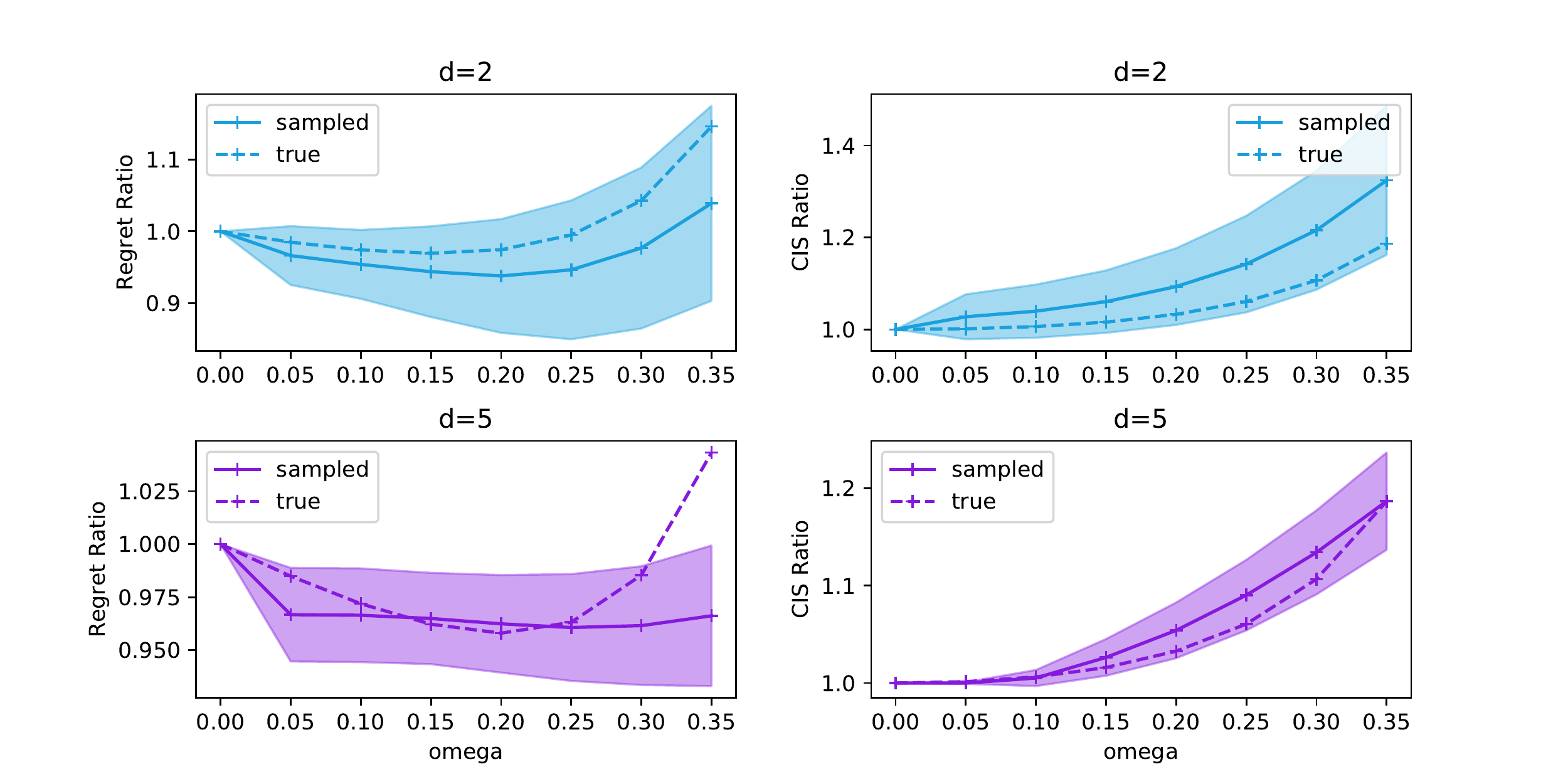}
		\caption{Performance Ratio in Classification, Model 2, when $k$ is Chosen Optimally  for $k$-NN}
		\label{fig:model2_same_k}
	\end{figure}
	Similar to \blue{the} regression setup, the theoretical value and empirical value trends are closed for the Regret ratio. Compared with the Regret ratio, the standard deviation of the CIS ratio is much larger due to the randomness of the two \blue{training} sample sets; thus, its trend is slightly away from the theoretical value.
	
	\subsubsection{$k$ is Chosen Optimally for each $\gamma$}\label{sec:numerical:diff}
	
	For all the three models in Section \ref{sec:numerical:same}, the performance ratios are also calculated when interpolated-NN uses the $k$ value, which \blue{optimizes} its testing performance. The formula for the performance ratio in this scenario can be found in Corollary \ref{coro}. The results are shown in Figures \ref{fig:regre3}, \ref{fig:model3} and \ref{fig:model2}. Similar to \blue{the results} in Section \ref{sec:numerical:same}, the empirical curves are mostly closed to the theoretical curves. 
	\begin{figure}[!ht]
		\centering
		\includegraphics[trim=0 20 50 0,clip,scale=0.65]{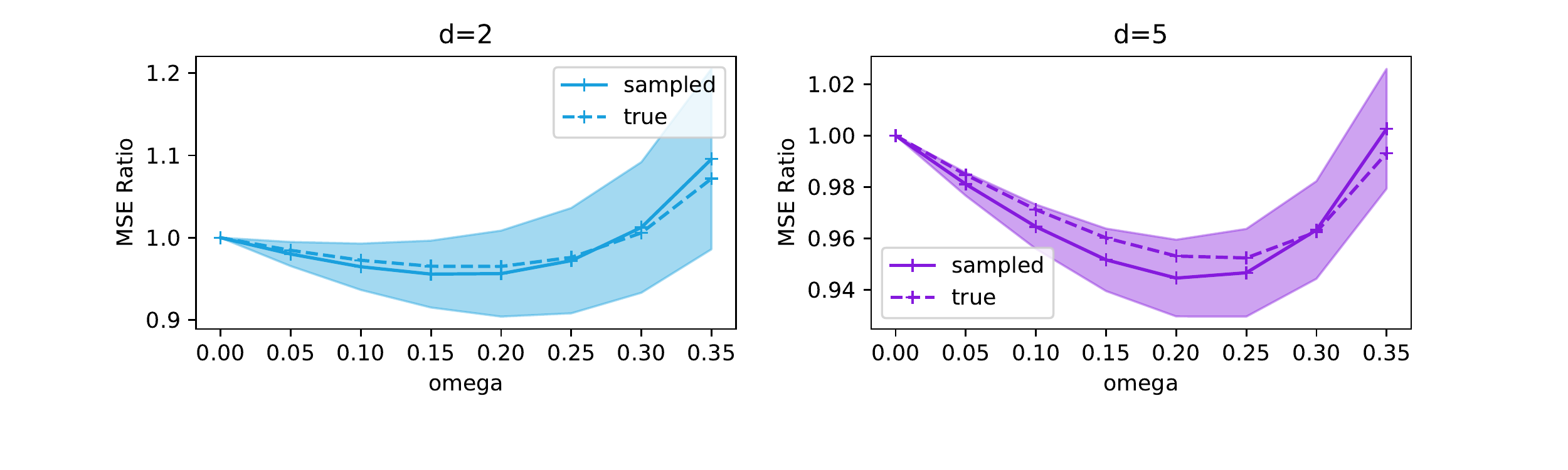}
		\caption{Performance Ratio in Regression when $k$ is Chosen Optimally for each $\gamma$}
		\label{fig:regre3}
	\end{figure}
	
	\begin{figure}[!ht]
		\centering
		\includegraphics[trim=0 0 50 0,clip,scale=0.65]{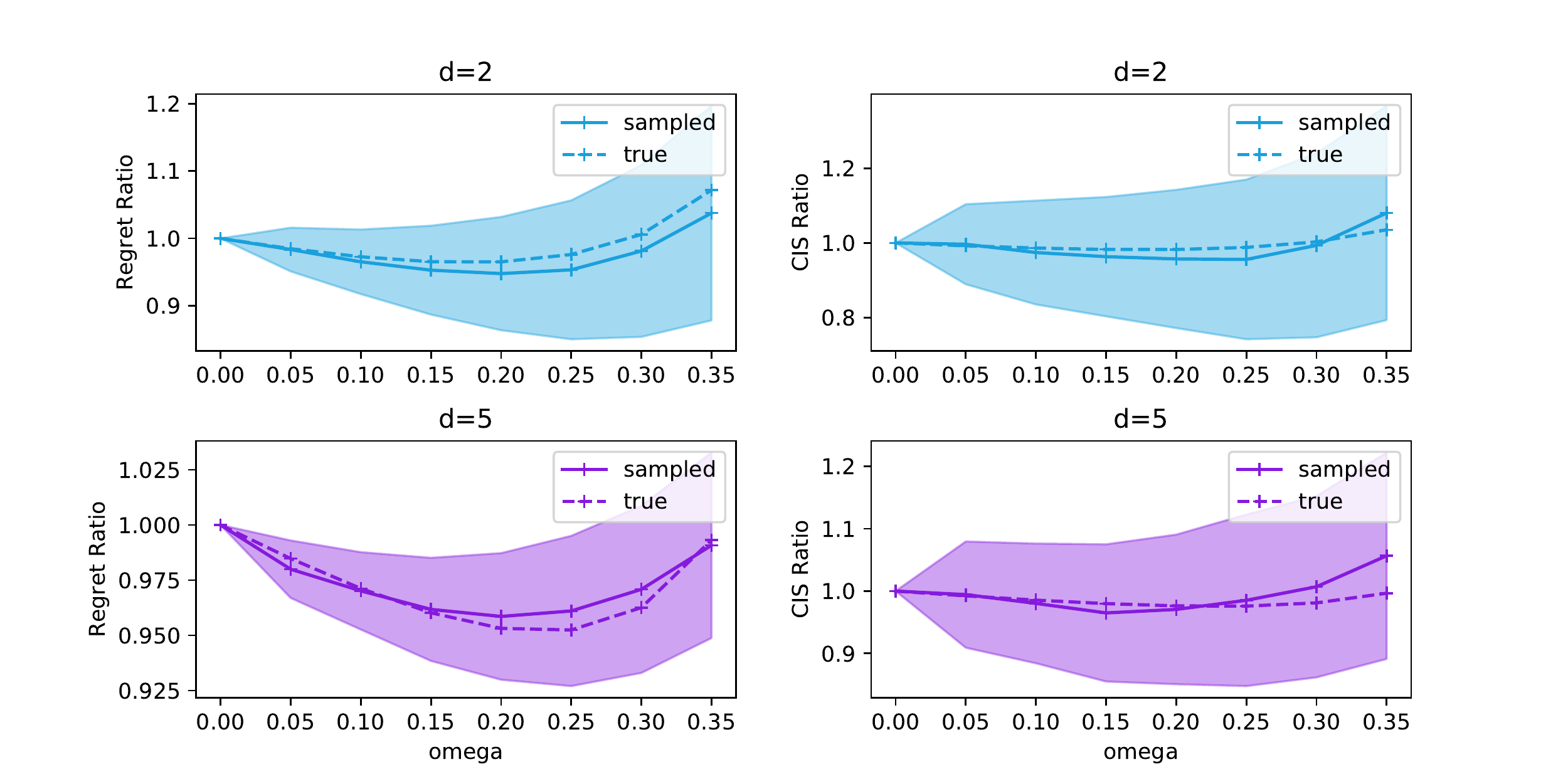}
		\caption{Performance Ratio in Classification, Model 1, when $k$ is Chosen Optimally for each $\gamma$}
		\label{fig:model3}
	\end{figure}
	\begin{figure}[!ht]
		\centering
		\includegraphics[trim=0 0 50 0,clip,scale=0.65]{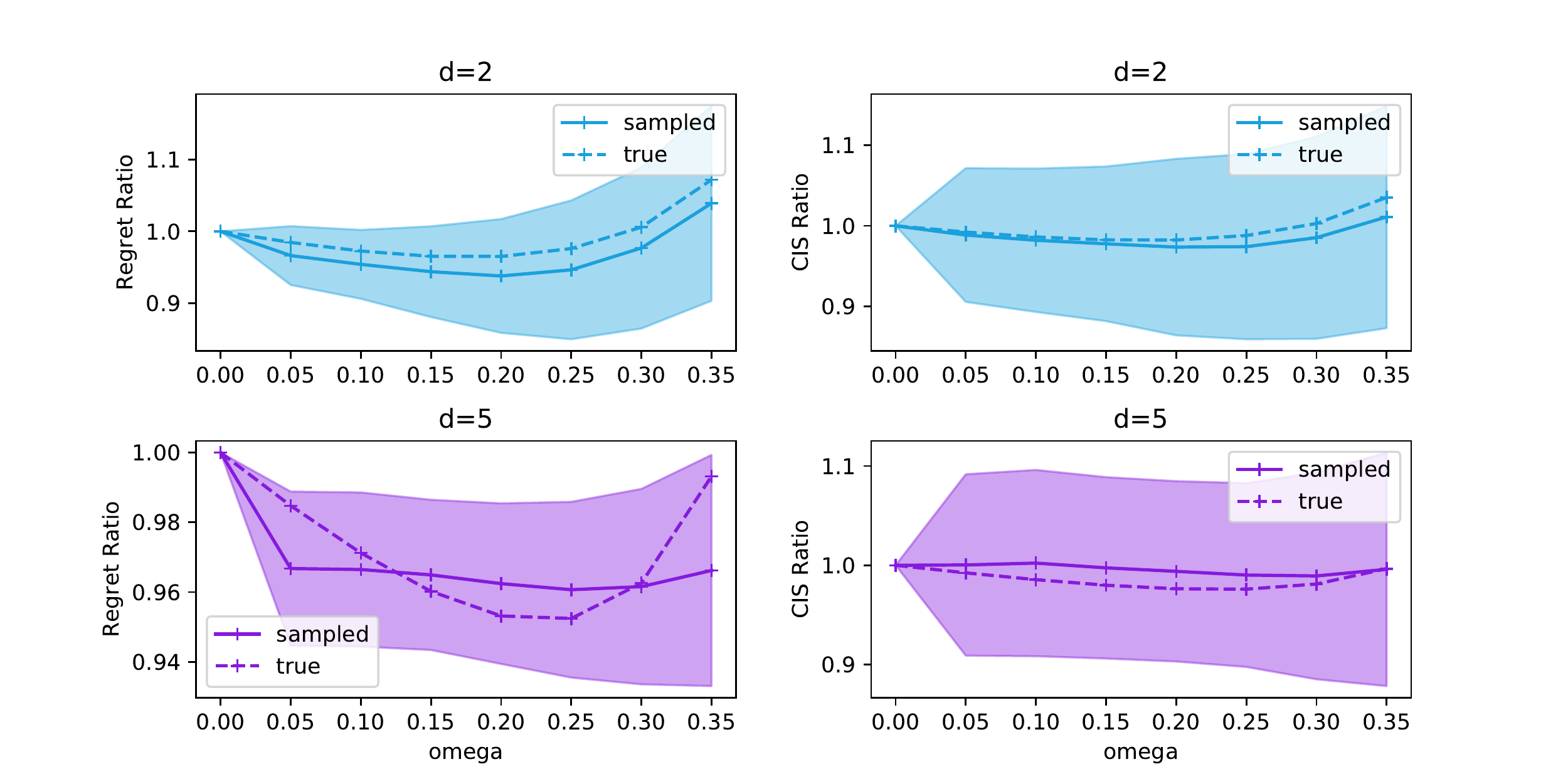}
		\caption{Performance Ratio in Classification, Model 2, when $k$ is Chosen Optimally for each $\gamma$}
		\label{fig:model2}
	\end{figure}	
	
	\subsection{Real Data Analysis}
	In \blue{the} real data experiments, we compare the classification accuracy of interpolated-NN with $k$-NN when $k$ is optimally tuned (via cross-validation) with respect to $\gamma$ values. 
	
	\begin{table}[ht]
		\centering
		\begin{tabular}{c|cccc}
			\hline
			Data & $d$ & Error ($\gamma=0$) & Error (best $\gamma/d$) & best $\gamma/d$\\\hline 
			Abalone & 7 & 0.22239 & 0.22007 & 0.3 \\
			HTRU2 & 8 & 0.02315 & 0.0226 & 0.2\\
			Credit & 23 & 0.1933 & 0.19287 & 0.05 \\
			Digit & 64 & 0.01745 & 0.01543 & 0.25\\
			MNIST & 784 & 0.04966 & 0.04656 & 0.05\\\hline
		\end{tabular}
		\caption{Prediction Error of $k$-NN, interpolated-NN under the best choice of $\gamma$, together with the value of the best $\gamma$ for interpolated-NN.}
		\label{tab:real}
	\end{table}
	
	Five data sets \blue{are} considered in this experiment. 
	The data set HTRU2 from \cite{lyon2016fifty} uses 17,897 samples with 8 continuous attributes to classify pulsar candidates. The data set Abalone contains 4,176 samples with seven attributes. Following \cite{wang2017analyzing}, we predict whether the number of rings is greater than 10. The data set Credit \citep{yeh2009comparisons} has 30,000 samples with 23 attributes and predicts whether the payment will be default in the next month given the current payment information. The built-in data set of digits in \textit{sklearn} \citep{scikit-learn} contains 1,797 samples of 8$\times$8 images. For images in MNIST are $26\times 26$, we will use part of it in our experiment. Both the data set of digit and MNIST have ten classes. Here for binary classification, we group 0 to 4 as the first and 5 to 9 as the second class.
	
	For each data set, a proportion of data is used for training, and the rest is reserved for testing. Note that the same testing data set is also used to tune the optimal choice of $k$. For Abalone, HTRU2, Credit, and Digit, we use 25\% data as training data and 75\% as testing data. For MNIST, we randomly choose 2000 samples as training data and 1000 as testing data, which is sufficient for our comparison.  The above experiment repeats 50 times, and the average testing error rate is summarized in Table \ref{tab:real}. 
	
	For all data sets, the best testing error among interpolated-NN's with different choices of $\gamma>0$ (column ``best $\gamma/d$'') is always smaller than the $k$-NN(column ``$\gamma=0$''), which verifies that the nearest neighbor algorithm \blue{can benefit from mild-level} interpolation.

	\section{Conclusion}\label{sec:conclusion}

	Our work precisely quantifies how data interpolation affects the performance of the nearest neighbor algorithms beyond the rate of convergence. For both regression and classification problems, the asymptotic performance ratios between interpolated-NN and $k$-NN converge to the same value, independent of data distribution, and depend on $d$ and $\gamma$ only. More importantly, when the interpolation level $\gamma/d$ is within a reasonable range, the interpolated-NN is strictly better than $k$-N, attaining both faster convergence and better stability performance.

	Classical learning frameworks oppose data interpolation as it believes that over-fitting means fitting the random noise rather than the model structures. However, in the interpolated-NN, the weight degenerating occurs only on a nearly zero-measure set, and thus there is only ``local over-fitting", which may not hurt the overall rate of convergence. Technically, through balancing the variance and bias, data interpolation potentially improves the overall performance. Essentially, our work precisely quantifies such a bias-variance balance. It is of great interest to investigate how our theoretical insights can be carried over to the real deep neural networks, leading to a complete picture of the double descent phenomenon.
	
	Finally, as mentioned in this paper, different algorithms (regression, nearest neighbors algorithm) may enjoy the benefit of interpolation in different ways; thus, \blue{an in-depth exploration on each algorithm is necessary to obtain} a comprehensive understanding of the double descent phenomenon. Nonetheless, although the \blue{mechanisms} behind these results are different, all these tell us that a carefully \blue{designed} statistical method can overcome \blue{harmful} over-fitting.

	\bibliographystyle{siamplain}
	\bibliography{VaRHDIS}
	\newpage
	\appendix

	The supplementary material is organized as follows. In Section \ref{sec:knn_wnn_compare}, we demonstrate a concrete example with an intuitive explanation of why interpolation does not hurt the convergence of interpolated-NN. Section B-H provide the proofs for the main theorems in the manuscript. 
	
	
	Section \ref{sec:ownn} extend the class of function interpolating weighting scheme to general class of functions which enjoys the benefit of variance-bias trade-off.



	\section{Variance-Bias Trade-off in NN Algorithms }\label{sec:knn_wnn_compare}
	In this section, we present an intuitive comparison between the interpolated-NN and traditional $k$-NN. For any weighted-NN regressor defined in (\ref{weightedNN}), 
	if the smooth condition $|\eta(x_1)-\eta(x_2)|\leq A\|x_1-x_2\|_2^{\alpha}$ holds for some $\alpha$ and $A$, then we have the following bias-variance decomposition (\citealp{belkin2018overfitting}):
	\begin{equation}\label{mse}
		\begin{split}
			&\mathbb{E}\big\{ (\widehat{\eta}_{k,n}(x)-\eta(x))^2 \big\}= \mathbb{E}\{(\mathbb{E}(\widehat{\eta}_{k,n}(x))-\eta(x))^2 \}+ \mathbb{E}\{(\widehat{\eta}_{k,n}(x)-\mathbb{E}(\widehat{\eta}_{k,n}(x)))^2 \}, \quad\mbox{where}\\
			&\mathbb{E}\{(\mathbb{E}(\widehat{\eta}_{k,n}(x))-\eta(x))^2 \}=\mbox{bias}^2\leq A^2\mathbb{E}\bigg[\sum_{i=1}^kW_i\|X^i-x\|^{\alpha}\bigg]^2,\\
			&\mathbb{E}\{(\widehat{\eta}_{k,n}(x)-\mathbb{E}(\widehat{\eta}_{k,n}(x)))^2 \}=\mbox{variance}=\sum_{i=1}^k \mathbb{E}\bigg[W_i^2(Y^i-\eta(X^i))^2 \bigg].
		\end{split}
	\end{equation}
	Several insights are developed based on the decomposition in (\ref{mse}). First, if ${\rm Var}(Y|X)$ (or $\mathbb{E}[(Y^i-\eta(X^i))^2 \mid X^i]$) is invariant among $X$ (e.g., the regression setting with i.i.d. noise $\epsilon_i$, or the classification setting with constant function $\eta$), then $k$-NN ($W_i\equiv 1/k$) represents the optimal weight choice which minimizes the variance term.  Second, if the weight assignment prioritizes closer neighbors, i.e., larger $W_i$ for smaller $\|X^i-x\|$, it will lead to a smaller value for the weighted average $\sum_{i=1}^kW_i\|X^i-x\|^{\alpha}$. In other words, interpolated-NN can achieve smaller upper bound for the bias term.
	Therefore, we argue that $k$-NN and interpolated-NN employ different strategies in reducing the upper bound of MSE. The former emphasizes reducing the variance, while the latter emphasizes reducing the bias. 
	
	The above intuitive arguments are well-validated by the following toy examples.

	We take 30 training samples $x_i=-5,-4,...,25$ with responses generated by the three choices: (1) $y=0*x+\epsilon$, (2) $y=x^2+0*\epsilon$, and (3) $y=(x-10)^2/8+5*\epsilon$ where $\epsilon\sim N(0,1)$. In other words, the mean function $\eta(x)$ are (1) $\eta(x)\equiv 0$, (2) $\eta(x)=x^2$, and (3) $\eta(x)=(x-10)^2/8$ respectively. The number of neighbors $k$ is chosen to be 10. Three different weighting schemes are considered: (1) $\phi(t)=1/k$, (2) $\phi(t)=1-\log(t)$, and (3) $\phi(t)=t^{-1}$, where the second and third choices are both interpolated weighting schemes, and the first choice is simply the traditional $k$-NN. 
	
	The regression estimators ($\widehat{\eta}_{k,n}(x)$) under different choices of $\phi$ and $\eta$ are shown in Figure \ref{toy}, along with the ``true'' curve which represent the underlying true $\eta(x)$. Note that we only present the  $\widehat{\eta}_{k,n}(x)$ within a smaller range of $x$ where the NN estimator does not suffer from the boundary effect. There are several observations from the results of these toy examples.
	
	\begin{figure}[ht]
		\begin{center}
			
			\includegraphics[scale=0.75]{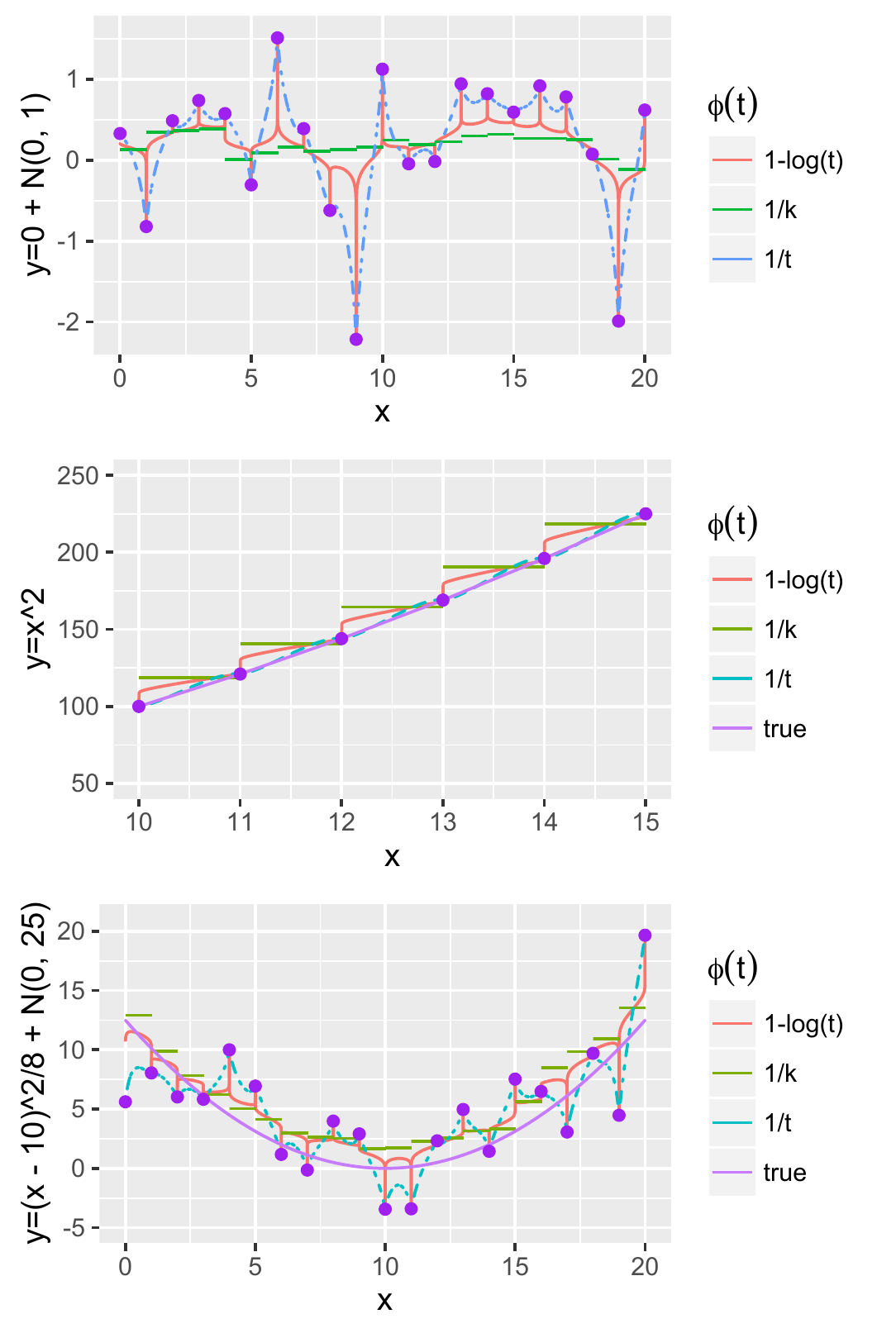}
			\caption{Three Toy Simulations: Upper: $\eta\equiv0$; Middle:  $\eta(x)=x^2$;
				Lower: $\eta(x)=(x-10)^{2}/8$.}
			\label{toy}
		\end{center}
	\end{figure}

	First of all, interpolated weight does ensure data interpolation. As $x$ gets closer to the any observed $x_i$, the estimator $\widehat{\eta}_{k,n}(x)$ is forced towards $y_i$. As a consequence, $\widehat{\eta}_{k,n}(x)$ is spiky for interpolated-NN. In contrast, the $k$-NN estimator is much more smooth.
	
	Secondly, different weighting schemes lead to different balance between bias and variance of $\widehat\eta_{k,n}$.
	In the first setting where $\eta\equiv 0$, the NN algorithm with any weighting scheme is unbiased, hence it corresponds to the extreme situation where bias is always 0.  The second model is noiseless (i.e., $\sigma^2$=0), thus it corresponds to the opposite extreme case where variance is always 0. 
	In the no-bias setting, $k$-NN performs the best, and interpolated-NN estimators lead to huge fluctuation. In the noiseless case, $k$-NN has the largest bias, and $\phi(t)=t^{-1}$ leads to the smallest bias. These observations are consistent with our arguments above, i.e., $k$-NN prioritizes minimizing the variance of the nearest neighbor estimator, while interpolated-NN prioritizes reducing the estimation bias instead. 
	For the comparison between the two different interpolated weighting schemes in this example, we comment that the faster $\phi(t)$ increases to infinite as $t\rightarrow0$, the smaller bias it will yield. Thus $\phi(t)=t^{-1}$ leads to smaller bias than $\phi(t)=1-\log(t)$, at the expense of larger estimation variance. The third model, $\eta(x)=(x-10)^2/8$, involves both noise and bias. From Figure \ref{toy}, the estimation of $\phi(t)=1/k$ tends to stay above the true $\eta$, i.e., high bias, due to the convexity of $\eta$. The estimators are more rugged for interpolated-NN but fluctuate along the true $\eta(x)$. In this case, it is difficult to claim which one is the best visually, but the trade-off phenomenon between bias and variance is clear.
	
	In conclusion, a faster decreasing $\phi$ leads to a smaller bias but a larger variance, and non-interpolated weights such as $k$-NN leads to a larger bias but a smaller variance.

	\section{Preliminary Proposition}
	This section provides an useful result when integrating c.d.f:
	\begin{proposition}\label{S.1}
		From Lemma S.1 in \cite{sun2016stabilized}, we have for any distribution function $G$,
		\begin{eqnarray*}
			\int_{\mathbb{R}} [G(-bu-a)-1_{\{ u<0 \}}]du &=& -\frac{1}{b}\left\{ a+\int_{\mathbb{R}}tdG(t)\right\},\\
			\int_{\mathbb{R}} u[G(-bu-a)-1_{\{ u<0 \}}]du &=& \frac{1}{b^2}\left\{  \frac{a^2}{2}+\frac{1}{2}\int_{\mathbb{R}}t^2dG(t) +a\int_{\mathbb{R}}tdG(t) \right\}.
		\end{eqnarray*}
	\end{proposition}
	\section{Proof of Theorem \ref{general}}
	This section demonstrates the proof of Theorem \ref{general}.
	
	Based on the same argument used in  \cite{CD14} and \cite{belkin2018overfitting}, conditional on $R_{k+1}(x)$, $X^1$ to $X^k$ are i.i.d. variables whose support is a ball centered at $x$ with radius $R_{k+1}(x)$, and as a consequence, $R_1(x)$ to $R_k(x)$ are conditionally independent given $R_{k+1}(x)$ as well. 
	
	The analysis of classification is very subtle especially when $\eta(x)$ is near $1/2$. Hence, we need to have the following partition over the space $\mathcal X$. Let $p=2k/n$. Denote $E=\{\exists R_i>r_{p},\; i=1,\ldots,k \}$, and $\mathbb{E}R_{k,n}(x)-R^*(x)$ as the excess risk. Then define
	\begin{equation*}
		\widetilde{\eta}_{k,n,\gamma}(x|R_{k+1})=\mathbb{E}[(R_1/R_{k+1})^{-\gamma}\eta(X^1)]/\mathbb{E}(R_1/R_{k+1})^{-\gamma},
	\end{equation*}
	as well as
	\begin{align*}
		\mathcal{X}^+_{p,\Delta}=\{x\in\mathcal{X}|\eta(x)>\frac{1}{2},\widetilde{\eta}(x)\geq &\frac{1}{2}+\Delta ,\forall R_{k+1}<r_{2p}(x)\},\\
		\mathcal{X}^-_{p,\Delta}=\{x\in\mathcal{X}|\eta(x)<\frac{1}{2},\widetilde{\eta}(x)\leq &\frac{1}{2}-\Delta , \forall R_{k+1}<r_{2p}(x)\},
	\end{align*}
	with the decision boundary area:
	\begin{equation*}
		\partial_{p,\Delta}=\mathcal{X}\setminus(\mathcal{X}^+_{p,\Delta}\cup\mathcal{X}^-_{p,\Delta}).
	\end{equation*}
	
	Given $	\partial_{p,\Delta}$, $\mathcal{X}^+_{p,\Delta}$, and $\mathcal{X}^-_{p,\Delta}$, similar with Lemma 8 in \cite{CD14}, the event of $g(x)\neq \widehat{g}_{k,n,\gamma}(x)$ can be covered as:
	\begin{eqnarray*}
		1_{\{ g(x)\neq \widehat{g}_{k,n,\gamma}(x) \}}\leq 1_{\{ x\in \partial_{p,\Delta}\}}+1_{\{  \max\limits_{i=1,\ldots,k}R_i\geq r_{2p} \}}+1_{\{ |\widehat{\eta}_{k,n,\gamma}(x)-\widetilde{\eta}(x|R_{k+1})|\geq \Delta \}}.
	\end{eqnarray*}
	
	When $\widetilde{\eta}(x|R_{k+1})>1/2$ and $x\in\mathcal{X}_{p,\Delta}^+$,  assume  $\widehat{\eta}_{k,n,\gamma}(x)<1/2$, then  $$\widetilde{\eta}_{k,n,\gamma}(x|R_{k+1})-\widehat{\eta}_{k,n,\gamma}(x) >\widetilde{\eta}_{k,n,\gamma}(x|R_{k+1})-1/2\geq \Delta.$$
	The other two events are easy to figure out.

	In addition, from the definition of \text{Regret}, assume $\eta(x)<1/2$, 
	\begin{eqnarray*}
		&&P(\widehat{g}(x)\neq Y|X=x)-\eta(x)\\&=& \eta(x)P(\widehat{g}(x)=0|X=x) + (1-\eta(x))P(\widehat{g}(x)=1|X=x) -\eta(x)\\
		&=&\eta(x)P(\widehat{g}(x)=g(x)|X=x)+ (1-\eta(x))P(\widehat{g}(x)\neq g(x)|X=x)-\eta(x)\\
		&=& \eta(x)-\eta(x)P(\widehat{g}(x)\neq g(x)|X=x)+ (1-\eta(x))P(\widehat{g}(x)\neq g(x)|X=x)-\eta(x)\\
		&=& (1-2\eta(x))P(\widehat{g}(x)\neq g(x)|X=x).
	\end{eqnarray*}
	Similarly, when $\eta(x)>1/2$, we have
	\begin{eqnarray*}
		P(\widehat{g}(x)\neq Y|X=x)-1+\eta(x)= (2\eta(x)-1)P(\widehat{g}(x)\neq g(x)|X=x).
	\end{eqnarray*}
	As a result, the \text{Regret} can be represented as
	\begin{eqnarray*}
		\text{Regret}(k,n,\gamma) &=& \mathbb{E}\left( |1-2\eta(X)|P(g(X)\neq \widehat{g}_{k,n,\gamma}(X)) \right).
	\end{eqnarray*}
	For simplicity, denote $p=k/n$. We then follow the proof of Lemma 20 of \cite{CD14}. Without loss of generality assume $\eta(x)>1/2$. Under A.1, A.2, and A.3, define
	\begin{eqnarray*}
		\Delta_0&=& \sup\limits_{x}|\widetilde{\eta}(x|R_{k+1})-\eta(x)|=O(r_{2p}^\alpha) =O(k/n)^{\alpha/d},\\
		\Delta(x)&=&|\eta(x)-1/2|,
	\end{eqnarray*}
	then$$\widetilde{\eta}(x|R_{k+1})\geq \eta(x)-\Delta_0=\frac{1}{2}+(\Delta(x)-\Delta_0), $$
	hence $x\in\mathcal{X}_{p,\Delta(x)-\Delta_0}^{+}$.

	When $\Delta(x)>\Delta_0$, under A.1 and A.2, the Regret can be upper bounded as
	\begin{eqnarray}\label{eqn:prob}
		&&P(\widehat{g}_{k,n,\gamma}(X)\neq Y)-P(g(X)\neq Y)\nonumber\\&\leq& 2\Delta(x) \bigg[P(r_{(k+1)}>v_{2p})+P\bigg( \sum_{i=1}^kW_iY(X^i)-\widetilde{\eta}(x|R_{k+1})>\Delta(x)-\Delta_0 \bigg)\bigg]\nonumber\\
		&\leq& \exp(-k/8)\nonumber\\
		&&+ 2\Delta(x) P\bigg( \sum_{i=1}^k(R_i/R_{k+1})^{-\gamma}Y(X^i)>(\widetilde{\eta}(x|R_{k+1})+\Delta(x)-\Delta_0)\sum_{i=1}^k(R_i/R_{k+1})^{-\gamma} \bigg)\nonumber\\
		&=& \exp(-k/8)\nonumber\\&&+ 2\Delta(x) P\bigg( \sum_{i=1}^k(R_i/R_{k+1})^{-\gamma}Y(X^i)-k\mathbb{E}(R_1/R_{k+1})^{-\gamma}\eta(X^1)\nonumber\\&&\qquad\qquad\qquad\qquad-(\widetilde{\eta}(x|R_{k+1})+\Delta(x)-\Delta_0)\sum_{i=1}^k\left[(R_i/R_{k+1})^{-\gamma}-\mathbb{E}(R_1/R_{k+1})^{-\gamma}\right]\nonumber\\&&\qquad\qquad\qquad\qquad>k(\widetilde{\eta}(x|R_{k+1})+\Delta(x)-\Delta_0)\mathbb{E}(R_1/R_{k+1})^{-\gamma}-k\mathbb{E}(R_1/R_{k+1})^{-\gamma}\eta(X^1) \bigg)\nonumber\\
		&=& \nonumber\exp(-k/8)+ 2\Delta(x) P\bigg\{ \sum_{i=1}^k(R_i/R_{k+1})^{-\gamma}(Y(X^i)-\widetilde{\eta}(x|R_{k+1}))\\&&\qquad\qquad\qquad\qquad\qquad\qquad-(\Delta(x)-\Delta_0)\left(\sum_{i=1}^k(R_i/R_{k+1})^{-\gamma}-k\mathbb{E}(R_1/R_{k+1})^{-\gamma}\right)\\&&\qquad\qquad\qquad\qquad\qquad\qquad>k(\Delta(x)-\Delta_0)\mathbb{E}(R_1/R_{k+1})^{-\gamma}\bigg\}\nonumber.
	\end{eqnarray}
	Since
	\begin{eqnarray*}
		&&\mathbb{E}\sum_{i=1}^k(R_i/R_{k+1})^{-\gamma}(Y(X^i)-\widetilde{\eta}(x|R_{k+1}))=0,\\
		&&\mathbb{E}(\Delta(x)-\Delta_0)\left(\sum_{i=1}^k(R_i/R_{k+1})^{-\gamma}-k\mathbb{E}(R_1/R_{k+1})^{-\gamma}\right)=0,
	\end{eqnarray*}
	we can use Markov inequality to the power of $\kappa(\beta)$ to bound the probability in (\ref{eqn:prob}). Denote $$Z_i(x)=\left(\frac{R_i}{R_{k+1}}\right)^{-\gamma}(Y(X^i)-\widetilde{\eta}(x|R_{k+1}))-(\Delta(x)-\Delta_0)\left(\frac{R_i}{R_{k+1}}\right)^{-\gamma}+(\Delta(x)-\Delta_0)\mathbb{E}\left(\frac{R_1}{R_{k+1}}\right)^{-\gamma}$$
	for simplicity. Note that
	\begin{eqnarray*}
		Var(Z_1(x))&=& O(\Delta(x)-\Delta_0).
	\end{eqnarray*}

	For different settings of $\beta$ and $\gamma$, the following steps have the same logic but different details:
	
	\paragraph{Case 1: $\beta< 2$ and $\gamma<d/3$:} 
	Considering the problem that the upper bound can be much greater than 1 when $\Delta(x)$ is small, we define $\Delta_i=2^i\Delta_0$, taking $i_0 = \min\{ i\geq 1|\; (\Delta_i-\Delta_0)^{2}>1/k \}$. In this situation, since $\mathbb{E}Z_1^3(x)<\infty$, we can adopt non-uniform Berry-Essen Theorem for the proof:
	\begin{eqnarray*}
		&&P(\widehat{g}_{k,n,\gamma}(X)\neq Y)-P(g(X)\neq Y)\nonumber\\
		&=&\mathbb{E}(R_{k,n}(X)-R^*(X))1_{\{\Delta(X)\leq\Delta_{i_0}\}}+\mathbb{E}(R_{k,n}(X)-R^*(X))1_{\{\Delta(X)>\Delta_{i_0}\}}\nonumber\\
		&\leq& 2\Delta_{i_0} P(\Delta(X)\leq \Delta_{i_0})+\exp(-k/8)
		+4\mathbb{E}\left[\Delta(X)1_{\{\Delta_{i_0}<\Delta(X)\}}   \bar{\Phi}\left( \frac{\sqrt{k}(\Delta(x)-\Delta_0)}{\sqrt{Var(Z_1(X)|X)}} \right)\right]\label{eqn:part1}\\&&+4\mathbb{E}\left[\Delta(X)1_{\{\Delta_{i_0}<\Delta(X)\}} \frac{c_1}{\sqrt{k}}\frac{1}{1+k^{3/2} (\Delta(x)-\Delta_0)^3 } \right].\label{eqn:part2}
	\end{eqnarray*}	
	Due to $\beta$-margin condition, the two terms from Berry-Essen Theorem in \blue{the above inequality} become
	\begin{eqnarray*}
		&&\mathbb{E}\left[\Delta(X)1_{\{\Delta_{i}<\Delta(X)<\Delta_{i+1}\}} \frac{c_1}{\sqrt{k}}\frac{1}{1+k^{3/2} (\Delta(x)-\Delta_0)^3 } \right]\\
		&\leq& \Delta_{i+1}^{\beta+1} \frac{c_1}{\sqrt{k}}\frac{1}{1+k^{3/2}(\Delta_i-\Delta_0)^3}\\
		&=& \frac{1}{k^{(\beta+1)/2}}(\sqrt{k}\Delta_{i+1})^{\beta+1} \frac{c_1}{\sqrt{k}}\frac{1}{1+k^{3/2}(\Delta_i-\Delta_0)^3},
	\end{eqnarray*}
	and
	\begin{eqnarray*}
		&&\mathbb{E}\left[\Delta(X)1_{\{\Delta_{i_0}<\Delta(X)\}} \bar{\Phi}\left( \frac{\sqrt{k}(\Delta(x)-\Delta_0)}{\sqrt{Var(Z_1(X)|X)}} \right)\right]\\
		&\leq&\mathbb{E}\left[\Delta(X)1_{\{\Delta_{i_0}<\Delta(X)\}} \bar{\Phi}\left( c_3\sqrt{k}(\Delta(x)-\Delta_0) \right)\right]\\
		&\leq& c_4\frac{\Delta_{i+1}^{\beta+1}}{\sqrt{k}{(\Delta_i-\Delta_0)}}\exp\left(-c_3^2k(\Delta_i-\Delta_0)^2\right).
	\end{eqnarray*}
	The quantity $\exp(-c_3^2k(\Delta_i-\Delta_0)^2)$ is larger than 1 if $\Delta_i\leq c_5/\sqrt{k}$. When $\Delta_i>c_5/\sqrt{k}$, for some constant $c_5>0$,
	\begin{eqnarray*}
		\frac{  \frac{\Delta_{i+1}^{\beta+1}}{{(\Delta_i-\Delta_0)}}\exp\left(-c_3^2k(\Delta_i-\Delta_0)^2\right)}{\frac{\Delta_{i}^{\beta+1}}{{(\Delta_{i-1}-\Delta_0)}}\exp\left(-c_3^2k(\Delta_{i-1}-\Delta_0)^2\right)}
		&=& 2^{\beta+1}\frac{ 2^{i-1}-1  }{2^{i}-1} \frac{\exp\left(-c_3^2k(\Delta_{i}-\Delta_0)^2\right)  }{\exp\left(-c_3^2k(\Delta_{i-1}-\Delta_0)^2\right)}\\
		&\leq& 2^{\beta} \frac{\exp\left(-c_3^2k(\Delta_{i}-\Delta_0)^2\right)  }{\exp\left(-c_3^2k(\Delta_{i-1}-\Delta_0)^2\right)}<1/2.
	\end{eqnarray*}
	
	Therefore the sum of the excess risk can be bounded. When $\beta<2$,
	\begin{eqnarray*}
		&&\mathbb{E}\left[\Delta(X)1_{\{\Delta_{i}<\Delta(X)<\Delta_{i+1}\}} \frac{c_1}{\sqrt{k}}\frac{1}{1+k^{3/2} (\Delta(x)-\Delta_0)^3 } \right]\\
		&\leq&  O\left( \frac{1}{k^{(\beta+2)/2}} \right)\sum_{i\geq i_0}  [\sqrt{k}(\Delta_{i+1}-\Delta_i)]^{\beta-2} \\
		&\leq& O\left( \frac{1}{k^{(\beta+2)/2}} \right)\sum_{i\geq i_0}  (\sqrt{k}\Delta_{i_0})^{\beta-2}2^{i(\beta-2)} \\
		&=& O\left( \frac{1}{k^{(\beta+2)/2}} k^{(\beta-2)/2}\Delta_{i_0}^{\beta-2} \right)=O\left(  \frac{\Delta_{i_0}^{\beta-2}}{k^2}\right),
	\end{eqnarray*}	
	
	and
	\begin{eqnarray*}
		&&\mathbb{E}\left[\Delta(X)1_{\{\Delta_{i_0}<\Delta(X)\}}   \bar{\Phi}\left( \frac{\sqrt{k}(\Delta(x)-\Delta_0)}{\sqrt{Var(Z_1(X)|X)}} \right)\right]\\
		&\leq& O\left(\frac{1}{\sqrt{k}}\right) \sum_{i\geq i_0} \frac{\Delta_{i+1}^{\beta+1}}{{(\Delta_i-\Delta_0)}}\exp\left(-c_3^2k(\Delta_i-\Delta_0)^2\right)\\
		&=&  O\left(\frac{1}{\sqrt{k}}\right) \Delta_{i_0+1}^{\beta} \exp\left(-c_3^2k(\Delta_{i_0}-\Delta_0)^2\right).
	\end{eqnarray*}
	Recall that $\Delta_{i_0}>\Delta_0$ and $\Delta_{i_0}^2>1/k$, hence when $\Delta_{i_0}^2=O(1/k)$, we can obtain the minimum upper bound
	\begin{eqnarray*}
		P(\widehat{g}_{k,n,\gamma}(X)\neq Y)-P(g(X)\neq Y)\leq O(\Delta_{0}^{\beta+1}) + O\left( \left(\frac{1}{k}\right)^{(\beta+1)/2} \right).
	\end{eqnarray*}
	Taking $k\asymp(n^{\alpha/(2\alpha+d)})$, the upper bound becomes $O(n^{-\alpha(\beta+1)/(2\alpha+d)})$.
	
	\paragraph{Case 2: $\gamma<d/\kappa(\beta)$: } \blue{In this case, we have}
	\begin{eqnarray*}
		&&P(\widehat{g}_{k,n,\gamma}(X)\neq Y)-P(g(X)\neq Y)\\
		&=&\mathbb{E}(R_{k,n}(X)-R^*(X))1_{\{\Delta(X)\leq\Delta_{i_0}\}}+\mathbb{E}(R_{k,n}(X)-R^*(X))1_{\{\Delta(X)>\Delta_{i_0}\}}\\
		&\leq& 2\Delta_{i_0} P(\Delta(X)\leq \Delta_{i_0})+\exp(-k/8)
		\\&&+4\mathbb{E}\left[\Delta(X)1_{\{\Delta_{i_0}<\Delta(X)\}} \frac{ \mathbb{E}\left( \sum_{i=1}^kZ_i(X)\right)^{\kappa(\beta)} }{(\Delta(X)-\Delta_0)^{\kappa(\beta)}k^{\kappa(\beta)}\mathbb{E}^{\kappa(\beta)}(R_1/R_{k+1})^{-\gamma}}\right],
	\end{eqnarray*}	
	while for some constant $c_1>0$,
	\begin{eqnarray*}
		&&\mathbb{E}\left[\Delta(X)1_{\{\Delta_i<\Delta(X)\leq \Delta_{i+1}\}} \frac{ \mathbb{E}\left( \sum_{i=1}^kZ_i(X)\right)^{\kappa(\beta)} }{(\Delta(X)-\Delta_0)^{\kappa(\beta)}k^{\kappa(\beta)}\mathbb{E}^{\kappa(\beta)}(R_1/R_{k+1})^{-\gamma}}\right]\\
		&\leq&\mathbb{E}\left[\Delta(X)1_{\{\Delta_i<\Delta(X)\leq \Delta_{i+1}\}} \right]\frac{ \mathbb{E}\left( \sum_{i=1}^kZ_i(X)\right)^{\kappa(\beta)} }{(\Delta_i-\Delta_0)^{\kappa(\beta)}k^{\kappa(\beta)}\mathbb{E}^{\kappa(\beta)}(R_1/R_{k+1})^{-\gamma}}\\
		&\leq&\Delta_{i+1}P(\Delta(X)\leq \Delta_{i+1})\frac{ \mathbb{E}\left( \sum_{i=1}^kZ_i(X)\right)^{\kappa(\beta)} }{(\Delta_i-\Delta_0)^{\kappa(\beta)}k^{\kappa(\beta)}\mathbb{E}^{\kappa(\beta)}(R_1/R_{k+1})^{-\gamma}}\\
		&\leq& c_1\left(\frac{1}{k}\right)^{{\kappa(\beta)}/2} \Delta_{i+1}^{\beta+1}/(\Delta_i-\Delta_0)^{\kappa(\beta)},
	\end{eqnarray*}	
	where the last inequality is obtained from $\beta$-margin condition and  the assumption that $d-{\kappa(\beta)}\gamma>0$. Note that since $\kappa(\beta)>\beta+1$, for some constant $0<c<1$,
	\begin{equation}
		\frac{\Delta_{i+1}^{\beta+1}/(\Delta_i-\Delta_0)^{\kappa(\beta)}}{\Delta_{i}^{\beta+1}/(\Delta_{i-1}-\Delta_0)^{\kappa(\beta)}}<c.
	\end{equation}
	Therefore the sum of the excess risk for can be bounded, where
	\begin{eqnarray}\label{eqn:infinite}
		&&\mathbb{E}\left[\Delta(X)1_{\{\Delta_{i_0}<\Delta(X)\}} \frac{ \mathbb{E}\left( \sum_{i=1}^kZ_i(X)\right)^{\kappa(\beta)} }{(\Delta(X)-\Delta_0)^{\kappa(\beta)}k^{\kappa(\beta)}\mathbb{E}^{\kappa(\beta)}(R_1/R_{k+1})^{-\gamma}}\right]\nonumber\\
		&\leq& O\left(\frac{1}{k}\right)^{{\kappa(\beta)}/2} \sum_{i\geq i_0}\Delta_{i+1}^{\beta+1}/(\Delta_i-\Delta_0)^{\kappa(\beta)}\\
		&\leq& O\left(\Delta_{i_0}^{\beta+1}\right) .\nonumber
	\end{eqnarray}	
	Recall that $\Delta_{i_0}>\Delta_0$ and $\Delta_{i_0}^2>1/k$, hence when $\Delta_{i_0}^2=O(1/k)$, we can obtain the minimum upper bound
	\begin{eqnarray*}
		P(\widehat{g}_{k,n,\gamma}(X)\neq Y)-P(g(X)\neq Y)= O(\Delta_{0}^{\beta+1}) + O\left( \left(\frac{1}{k}\right)^{(\beta+1)/2} \right).
	\end{eqnarray*}
	Taking $k\asymp(n^{\alpha/(2\alpha+d)})$, the upper bound becomes $O(n^{-\alpha(\beta+1)/(2\alpha+d)})$.
	\section{Proof of Theorem \ref{thm:cis_rate}}
	This section provides the proof of Theorem \ref{thm:cis_rate}.
	
	For $CIS_{k,n}(\gamma)$, since we are considering two independent sets of data, it can be upper bounded by two times misclassification error:
	\begin{eqnarray*}
		CIS_{k,n}(\gamma)&=&2P(\hat{g}_1(X)\neq g(X), \hat{g}_2(X)=g(X) )\\
		&=& 2\mathbb{E}\left[ 1-P(\hat{g}_2(X)=g(X)|X) \right]P(\hat{g}_1(X)\neq g(X)|X)\\
		&\leq& 2\mathbb{E} P(\hat{g}_1(X)\neq g(X)|X)\\
		&=&2P(\hat{g}(X)\neq g(X)).
	\end{eqnarray*}
	Following the same procedures as in Theorem \ref{general}, when $\beta<2$ and $\gamma<d/3$, we have
	\begin{eqnarray*}
		&&P(\hat{g}_{k,n,\gamma}(X)\neq g(X))\\
		&=&\mathbb{E} P(\hat{g}_{k,n,\gamma}(X)\neq g(X))1_{\{ \Delta(X)\leq\Delta_{i_0} \}}+\mathbb{E} P(\hat{g}_{k,n,\gamma}(X)\neq g(X))1_{\{ \Delta(X)>\Delta_{i_0} \}}\\
		&\leq& P(\Delta(X)\leq\Delta_{i_0}) + \exp(-k/8)+\mathbb{E} \left[ 1_{\{\Delta_{i_0}<\Delta(X)\}} \bar{\Phi}\left( \frac{\sqrt{k}(\Delta(X)-\Delta_0)}{\sqrt{Var(Z_1(X)|X)}} \right) \right]\\
		&&+\mathbb{E}\left[1_{\{\Delta_{i_0}<\Delta(X)\}} \frac{c_1}{\sqrt{k}}\frac{1}{1+k^{3/2}(\Delta(X)-\Delta_0)^3} \right].
	\end{eqnarray*}
	The two terms from Berry-Essen Theorem becomes
	\begin{eqnarray*}
		\mathbb{E}\left[1_{\{\Delta_i<\Delta(X)<\Delta_{i+1}\}} \frac{c_1}{\sqrt{k}}\frac{1}{1+k^{3/2}(\Delta(X)-\Delta_0)^3} \right]\leq (\sqrt{k}\Delta_{i+1})^{\beta}\frac{1}{k^{\beta/2}}\frac{c_1}{\sqrt{k}}\frac{1}{1+k^{3/2}(\Delta_i-\Delta_0)^3},
	\end{eqnarray*}
	and
	\begin{eqnarray*}
		\mathbb{E} \left[ 1_{\{\Delta_{i_0}<\Delta(X)\}} \bar{\Phi}\left( \frac{\sqrt{k}(\Delta(X)-\Delta_0)}{\sqrt{Var(Z_1(X)|X)}} \right) \right]\leq c_4\frac{\Delta_{i+1}^\beta}{\sqrt{k}(\Delta_i-\Delta_0)}\exp(-c_3^2k(\Delta_i-\Delta_0)^2).
	\end{eqnarray*}
	Thus summing up $\Delta_i<\Delta(X)<\Delta_{i+1}$ for all $i>i_0$, we have
	\begin{equation*}
		CIS_{k,n}(\gamma)\leq O(\Delta_0^{\beta})+O\left( (1/k)^{\beta/2} \right).
	\end{equation*}
	When $k\asymp n^{\alpha/(2\alpha+d)}$, we have $CIS=O(n^{-\alpha\beta/(2\alpha+d)})$. The proof for $\beta\geq 2$ s similar.
	\section{Proof of Theorem \ref{thm:main} }\label{sec:thm:proof}
	
	This section is the proof of Theorem \ref{thm:main}.
	
	To prove Theorem \ref{thm:main}, we first prove the following theorem, then represents the multiplicative constants as functions of $\gamma$.
	\begin{theorem}\label{thm:main1}
		Assume $d-3\gamma\geq C>0$ for some constant $C$. For regression, suppose that assumptions A.1', A.2', A.5', and A.6' hold. If $k$ satisfies $n^{\delta}\leq k \leq n^{1-4\delta/d}$ for some $\delta>0$, we have
		\begin{eqnarray*}
			\text{MSE}(k,n,\gamma) &=&  \underbrace{k\mathbb{E}\left[\frac{(R_{1}(X)/R_{k+1}(X))^{-2\gamma}}{ (\sum_{i=1}^k (R_{i}(X)/R_{k+1}(X))^{-\gamma})^2}\sigma^2(X) \right]}_{Variance}\\&&+ \underbrace{k^2\mathbb{E}\left(a^2(X) \mathbb{E}^2\left[\frac{R_1(X)^2(R_1(X)/R_{k+1}(X))^{-\gamma}}{\sum_{i=1}^k(R_i(X)/R_{k+1}(X))^{-\gamma}}\bigg| X\right]\right)}_{Bias}+Remainder,
		\end{eqnarray*}
		where $Remainder=o(\text{MSE}(k,n,\gamma))$.
		
		For classification, under A.1' to A.4', the excess risk w.r.t. $\gamma$ becomes 
		\begin{eqnarray*}
			\text{Regret}(k,n,\gamma)=\underbrace{\frac{1}{4k}B_1\mathbb{E}s_{k,n,\gamma}^2(X)}_{Variance}+\underbrace{\int_{S}\frac{f(x_0)}{\|\dot{\eta}(x_0)\|} t_{k,n,\gamma}^2(x_0) d\text{Vol}^{d-1}(x_0)}_{Bias}+Remainder,
		\end{eqnarray*}
		where \begin{eqnarray*}
			Remainder&=&o(\text{Regret}(k,n,\gamma)),\\
			B_1&=&\int_{S}\frac{f(x_0)}{\|\dot{\eta}(x_0)\|}d\text{Vol}^{d-1}(x_0),\\
			s^2_{k,n,\gamma}(x)&=&\frac{\mathbb{E} ({R}_1/{R}_{k+1})^{-2\gamma} }{\mathbb{E}^2({R}_1/{R}_{k+1})^{-\gamma}},\\
			t_{k,n,\gamma}(x)&=& \frac{\mathbb{E}{R}_1^{2} ({R}_1/{R}_{k+1})^{-\gamma} }{\mathbb{E}({R}_1/{R}_{k+1})^{-\gamma}}.
		\end{eqnarray*} 
		
	\end{theorem}
	\subsection{Regression}
	
	Rewrite the interpolated-NN estimate at $x$ given the distance to the $k+1$th neighbor $R_{k+1}$, interpolation level $\gamma$ as $$S_{k,n,\gamma}(x,R_{k+1}) =\sum_{i=1}^k W_iY_i,$$
	where the weighting scheme is defined as
	\begin{eqnarray*}
		W_i=\frac{(R_i/R_{k+1})^{-\gamma}}{\sum_{i=1}^k (R_i/R_{k+1})^{-\gamma}}.
	\end{eqnarray*}
	For regression, we decompose \text{MSE} into bias square and variance, where
	\begin{eqnarray*}
		\mathbb{E}[(S_{k,n,\gamma}(x,R_{k+1})-\eta(x))^2|x]&=&\mathbb{E} \left[\sum_{i=1}^k W_i(\eta(X^i)-\eta(x))\right]^2+\mathbb{E} \left[\sum_{i=1}^k W_i(Y_i-\eta(X^i))\right]^2,
	\end{eqnarray*}
	in which the bias square can be rewritten as
	\begin{eqnarray*}
		\mathbb{E} \left[\sum_{i=1}^k W_i(\eta(X^i)-\eta(x))\right]^2&=& k\mathbb{E}(W_1(\eta(X^1)-\eta(x)))^2 + (k^2-k)\mathbb{E}^2(W_1(\eta(X^1)-\eta(x))),
	\end{eqnarray*}
	and the variance can be approximated as
	\begin{eqnarray*}
		\mathbb{E} \left[\sum_{i=1}^k W_i(Y_i-\eta(X^i))\right]^2= k\mathbb{E}W_1^2\sigma(X^1)^2
		=k\sigma(x)^2\mathbb{E}W_1^2+o.
	\end{eqnarray*}
	Following a procedure similar as \textit{Step 1} for classification, i.e., use Taylor expansion to approximate the bias square, we obtain that for some function $a$, the bias becomes
	\begin{eqnarray*}
		\mathbb{E}W_1(\eta(X^1)-\eta(x)) = a(x) \mathbb{E}W_1^2R_1^2+o.
	\end{eqnarray*}
	As a result, the \text{MSE} of interpolated-NN estimate given $x$ becomes,
	\begin{eqnarray*}
		\mathbb{E}[(S_{k,n,\gamma}(x,R_{k+1})-\eta(x))^2|x] = k\sigma(x)^2\mathbb{E}W_1^2+k^2a(x)^2\mathbb{E}^2W_1^2R_1^2 +o.
	\end{eqnarray*}
	Finally we integrate \text{MSE} over the whole support.
	
	\subsection{Classification}\label{sec:classification_proof}
	The main structure of the proof follows \cite{samworth2012optimal}. As the whole proof is long, we provide a brief summary in Section \ref{sec:summary} to describe things we will do in each step, then in Section \ref{sec:detail} we will present the details in each step.
	\subsubsection{Brief Summary}\label{sec:summary}
	\textit{Step 1}: denote i.i.d. random variables $Z_i(x,R_{k+1})$ for $i=1,\ldots,k$ where	\begin{eqnarray*}
		Z_i(x,R_{k+1})=\frac{ (R_i/R_{k+1})^{-\gamma}(Y(X^i)-1/2) }{\mathbb{E}(R_i(x)/R_{k+1}(x))^{\gamma}},
	\end{eqnarray*} 
	then the probability of classifying as 0 becomes $$P(S_{k,n,\gamma}(x)<1/2)=P\left(\sum_{i=1}^k Z_i(x,R_{k+1})<0\right).$$ The mean and variance of $Z_i(x,R_{k+1})$ can be obtained through Taylor expansion of $\eta$ and density function of $x$:
	\begin{eqnarray*}
		\mathbb{E}(Z_1(x,R_{k+1}))&=&\eta(x)-\frac{1}{2}+a(x)\frac{\mathbb{E}R_1^2(R_1/R_{k+1})^{-\gamma}}{\mathbb{E}(R_1/R_{k+1})^{-\gamma}}+o\\
		Var(Z_1(x,R_{k+1}))&=&\frac{1}{4}\frac{\mathbb{E}(R_1/R_{k+1})^{-2\gamma}}{\mathbb{E}^2(R_1/R_{k+1})^{-\gamma}}+o,
	\end{eqnarray*}
	for some function $a$. The smoothness conditions are assumed in A.4 and A.5. 
	
	Note that on the denominator of $Z_i$, there is an expectation $\mathbb{E}(R_i(x)/R_{k+1}(x))^{\gamma}$. From later calculation in Corollary \ref{coro}, the value of this expectation in fact has little changes given or without a condition of $R_{k+1}$, and it is little affected by $x$ either.
	
	\textit{Step 2}: 
	One can rewrite \text{Regret} as
	$$ \int_{\mathbb{R}^d} \left( P\left( \sum_{i=1}^k W_iY_i\leq \frac{1}{2} \right)-1_{\{ \eta(x)<1/2 \}}  \right)d\bar{P}(x).$$
	
	From Assumption A.2, A.4, the region where $\widehat{\eta}$ is likely to make a wrong prediction is near $\{ x|\eta(x)=1/2 \}$, thus we use tube theory to transform the integral of \text{Regret} over the $d$-dimensional space into a tube, i.e.,
	\begin{eqnarray*}
		&&\int_{\mathbb{R}^{d}}\left(P\left(\sum_{i=1}^kW_iY_i\leq \frac{1}{2}\right)-1\{ \eta(x)<1/2 \} \right)d\bar{P}(x)\\
		&=& \{1+o(1)\}\int_{\mathcal{S}}\int_{-\epsilon}^{\epsilon}t\|\dot{\Psi}(x_0)\|\left( P(S_{k,n}(x_0^t)<1/2)-1_{\{t<0\}}\right)dtd\text{Vol}^{d-1}(x_0) +o.
	\end{eqnarray*}
	The term $\epsilon$ will be defined in detail later. Basically, when $\epsilon$ is within a suitable range, the integral over $t$ will not depend on $\epsilon$ asymptotically.
	
	\textit{Step 3}: given $R_{k+1}$ and $x$, the nearest $k$ neighbors are i.i.d. random variables distributed in $B(x,R_{k+1})$, thus we use non-uniform Berry-Esseen Theorem to get the Gaussian approximation of the probability of wrong prediction:
	\begin{eqnarray*}
		&&\int_{\mathcal{S}}\int_{-\epsilon}^{\epsilon}t\|\dot{\Psi}(x_0)\| \left(P(S_{k,n}(x_0^t)<1/2)-1_{\{t<0\}}\right)dtd\text{Vol}^{d-1}(x_0) \\&=&
		\int_{\mathcal{S}}\int_{-\epsilon}^{\epsilon}t\|\dot{\Psi}(x_0)\|\mathbb{E}_{R_{k+1}} \left( \Phi\left( \frac{-k\mathbb{E}Z_1(x_0^t,R_{k+1})}{\sqrt{kVar(Z_1(x_0^t,R_{k+1}))}} \right)-1_{\{t<0\}} \right)dtd\text{Vol}^{d-1}(x_0)+o.
	\end{eqnarray*}
	
	\textit{Step 4}: take expectation over all $R_{k+1}$, and integral the Gaussian probability over the tube to obtain
	\begin{eqnarray*}
		&&\int_{\mathcal{S}}\int_{-\epsilon}^{\epsilon}t\|\dot{\Psi}(x_0)\|\mathbb{E}_{R_{k+1}} \left( \Phi\left( \frac{-k\mathbb{E}Z_1(x_0^t,R_{k+1})}{\sqrt{kVar(Z_1(x_0^t,R_{k+1}))}} \right)-1_{\{t<0\}} \right)dtd\text{Vol}^{d-1}(x_0)\\&=&\int_{\mathcal{S}}\int_{\mathbb{R}} t\|\dot{\Psi}(x_0)\|\left(\Phi\left(-\frac{t\|\dot{\eta}(x_0)\|}{\sqrt{s_{k,n,\gamma}^2/k}} -\frac{ \mathbb{E}(R_1/R_{k+1})^{-\gamma}a(x_0^t)R_1^2}{\sqrt{s_{k,n,\gamma}^2/k}}\right)-1_{\{t<0\}}\right)dtd\text{Vol}^{d-1}(x_0)\\&&+o\\
		&=& \frac{B_1}{4k}\frac{ \mathbb{E}(R_1/R_{k+1})^{-2\gamma} }{\mathbb{E}^2(R_1/R_{k+1})^{-\gamma}} + \int_{S}\frac{\|\dot{\Psi}(x_0)\|}{\|\dot{\eta}(x_0)\|^2}a^2(x_0) \frac{  \mathbb{E}^2(R_1/R_{k+1})^{-\gamma}R_1^{2}}{\mathbb{E}^2(R_1/R_{k+1})^{-\gamma}} d\text{Vol}^{d-1}(x_0)+o
		\\&=&\frac{1}{8k}B_1\mathbb{E}s_{k,n,\gamma}^2+\int_{S}\frac{\|\dot{\Psi}(x_0)\|}{2\|\dot{\eta}(x_0)\|^2} a^2(x_0)t_{k,n,\gamma}^2 d\text{Vol}^{d-1}(x_0)+o.
	\end{eqnarray*}

	\subsubsection{Details}\label{sec:detail}
	Denote  $a_d$ is the Euclidean ball volume parameter $$a_d=\text{Vol}(B(0,1))=(\pi/2)^{d/2}/\Gamma(d/2+1).$$

	Define $p=k/n$ and $r_{2p}=\sup\limits_{x}\mathbb{E}R_{2k}(x)$. Denote $E$ be the set that there exists $R_i$ such that $R_{i}>r_{2p}$, then for some constant $c>0$,
	\begin{eqnarray*}
		r_{2p}=\frac{c}{a_d^{1/d}c_0^{1/d}}\left( \frac{2k}{n}\right)^{1/d}.
	\end{eqnarray*} Hence from Claim A.5 in \cite{belkin2018overfitting}, there exist $c_1$ and $c_2$ satisfying$$P(E)\leq c_1k\exp(-c_2k).$$
	
	\textit{Step 1:} in this step, we figure out the i.i.d. random variable in our problem, and calculate its mean and variance given $x$. 
	
	Denote
	\begin{eqnarray}\label{eqn:Z}
		Z_i(x,R_{k+1})=\frac{ (R_i/R_{k+1})^{-\gamma}(Y(X^i)-1/2) }{\mathbb{E}(R_i/R_{k+1})^{\gamma}},
	\end{eqnarray} 
	then the dominant part we want to integrate becomes
	\begin{eqnarray*}
		&&P\left( {S}_{k,n}(x,R_{k+1})\leq \frac{1}{2} \right) \\&=& P\left( \sum_{i=1}^k(R_i/R_{k+1})^{-\gamma}(Y(X^i)-1/2)<0 \bigg| R_{k+1}\right)\\
		&=& P\left( \frac{\sum_{i=1}^k Z_i(x,R_{k+1})-k\mathbb{E}Z_1(x,R_{k+1})}{\sqrt{kVar(Z_1(x,R_{k+1}))}}< \frac{ -k\mathbb{E}Z_1(x,R_{k+1})}{\sqrt{kVar(Z_1(x,R_{k+1}))}}\bigg| R_{k+1}\right).
	\end{eqnarray*}
	Therefore, one can adopt non-uniform Berry-Essen Theorem to approximate the probability using normal distribution. Unlike \cite{samworth2012optimal} in which $\mathbb{E}Y(X^i)$ is calculated, since the i.i.d. item in non-uniform Berry-Essen Theorem is $Z$ rather than $Y$, we now calculate mean and variance of $Z$. Under $R_{k+1}$,
	\begin{eqnarray*}
		\mu_{k,n,\gamma}(x,R_{k+1}):=\mathbb{E}Z_1(x,R_{k+1})&=& \frac{  \mathbb{E}(R_1/R_{k+1})^{-\gamma}(Y(X^1)-1/2) }{\mathbb{E}(R_{1}/R_{k+1})^{-\gamma}} \\
		&=&\frac{ \mathbb{E}(R_1/R_{k+1})^{-\gamma}(\eta(X^1)-1/2) }{\mathbb{E}(R_{1}/R_{k+1})^{-\gamma}},
	\end{eqnarray*}
	and
	\begin{eqnarray*}
		\mathbb{E}Z_1^2(x,R_{k+1})&=& \frac{ \mathbb{E}(R_1/R_{k+1})^{-2\gamma}(Y(X^1)-1/2)^2  }{\mathbb{E}^2(R_{1}/R_{k+1})^{-\gamma} } \\
		&=&\frac{   \mathbb{E}(R_1/R_{k+1})^{-2\gamma}}{4\mathbb{E}^2(R_{i}/R_{k+1})^{-\gamma}},\\\sigma^2_{k,n,\gamma}(x,R_{k+1})&:=&Var(Z_1(x,R_{k+1})).
	\end{eqnarray*}
	
	Then the mean and variance of $Z_i$ can be calculated as
	
	\begin{eqnarray*}
		&&{\mu}_{k,n}(x_0^t,R_{k+1}) =\mathbb{E}Z_1(x_0^t,R_{k+1})+\frac{1}{2}=\frac{ \mathbb{E}(R_1/R_{k+1})^{-\gamma}\eta(X^1)}{\mathbb{E}(R_1/R_{k+1})^{-\gamma}}\\
		&=&  \frac{\mathbb{E} (R_1/R_{k+1})^{-\gamma} \left(\eta(x_0^t)+ (X^1-x_0^t)^{\top} \dot{\eta}(x_0^t) + 1/2(X^1-x_0^t)^{\top}\ddot{\eta}(x_0^t) (X^1-x_0^t)\right) }{\mathbb{E}(R_1/R_{k+1})^{-\gamma}} +o(R_{k+1}^3)\\
		&=& \eta(x_0^t) + \frac{ \mathbb{E} (R_1/R_{k+1})^{-\gamma}(X^1-x_0^{t})^{\top}\dot{\eta}(x_0^t) }{\mathbb{E}(R_1/R_{k+1})^{-\gamma}} \\&&+ \frac{1}{2} \frac{\mathbb{E} (R_1/R_{k+1})^{-\gamma}  tr[\ddot{\eta}(x_0^t) \left( (X^1-x_0^t)(X^1-x_0^t)^{\top}\right)] }{\mathbb{E}(R_1/R_{k+1})^{-\gamma}}  + O(R_{k+1}^3).
	\end{eqnarray*}
	Fixing $R_1$ and $R_{k+1}$, denoting $f(\cdot|x_2,R)$ as the conditional density of $X$ given $x_2$ and $\|X-x_2\|=R_1$, we have
	\begin{eqnarray}\label{eqn:first_order}
		&&\mathbb{E}((X^1-x_0^{t})^{\top}\dot{\eta}(x_0^t) |R_1)\nonumber\\ &=& \int (x-x_0^t)^{\top}\dot{\eta}(x_0^t) f(x|x_0^t,R_{k+1})dx\nonumber\\
		&=& \int (x-x_0^t)^{\top}\dot{\eta}(x_0^t) [f(x_0^t|x_0^t,R_{1})+f'(x_0|x_0^t,R_{1})^{\top}(x-x_0^t)+o]dx\nonumber\\
		&=& 0+\int (x-x_0^t)^{\top}\dot{\eta}(x_0^t) f'(x_0^t|x_0^t,R_{1})^{\top}(x-x_0^t)dx+o\nonumber\\
		&=& tr\left( \dot{\eta}(x_0^t) f'(x_0^t|x_0^t,R_{1})^{\top} \int(x-x_0^t) (x-x_0^t)^{\top}dx \right)+o
	\end{eqnarray}
	and
	\begin{eqnarray}\label{eqn:second_order}
		&&tr\left(  \frac{1}{2}\ddot{\eta}(x_0^t)\mathbb{E} \left( (X^1-x_0^t)(X^1-x_0^t)^{\top}|R_{1}\right) \right)\nonumber\\
		&=& tr\left(  \frac{1}{2}\ddot{\eta}(x_0^t)\int (x-x_0^t)(x-x_0^t)^{\top} f(x|x_0^t,R_{1})dx\right)\nonumber\\
		&=& tr\left(  \frac{1}{2}\ddot{\eta}(x_0^t)\int (x-x_0^t)(x-x_0^t)^{\top} [f(x_0^t|x_0^t,R_{1})+f'(x_0|x_0^t,R_{1})^{\top}(x-x_0^t)+o]dx\right)\nonumber\\
		&=& tr\left(  \frac{f(x_0^t|x_0^t,R_{1})}{2}\ddot{\eta}(x_0^t)\int (x-x_0^t)(x-x_0^t)^{\top} dx\right)+o,
	\end{eqnarray}
	Then taking function $a(x)$ for $x$ such that 
	\begin{eqnarray*}
		a(x_0^t)R_{1}^2 &=& \mathbb{E}((X^1-x_0^{t})^{\top}\dot{\eta}(x_0^t) |R_1)+ tr\left(  \frac{1}{2}\ddot{\eta}(x_0^t)\mathbb{E} \left( (X^i-x_0^t)(X^i-x_0^t)^{\top}|R_1\right) \right)+o.
	\end{eqnarray*}
	The difference caused by the value of $R_1$ is only a small order term.
	
	Finally,
	\begin{eqnarray*}
		\mu_{k,n,\gamma}(x_0^t,R_{k+1})
		&=&\eta(x_0) + t\|\dot{\eta}(x_0)\| + a(x_0)t_{k,n,\gamma}(R_{k+1})+o.
	\end{eqnarray*}

	\textit{Step 2:} in this step we construct a tube based on the set $\mathcal{S}=\{x|\eta(x)=1/2\}$, then figure out that the part of \text{Regret} outside this tube is a remainder term.

	Assume $\epsilon_{k,n}$ satisfies $s_{k,n,\gamma}=o(\epsilon_{k,n})$ and $\epsilon_{k,n}=o(s_{k,n,\gamma}k^{1/2})$, then the residual terms throughout the following steps will be $o(s_{k,n,\gamma}^2+t_{k,n,\gamma}^2)$. Hence although the choice of $\epsilon_{k,n}$ is different among choices of $k$ and $n$,  this does not affect the rate of \text{Regret}. Note that we ignore the arguments $x$ and $R_{k+1}$ as from A.2 and A.4, $s_{k,n,\gamma}^2(x,R_{k+1})\asymp 1/k$ and $t_{k,,n,\gamma}(x,R_{k+1})\asymp(k/n)^{2/d}$ for all $x$ while $R_{k+1}\asymp (k,n)^{2/d}$ in probability.

	By \cite{samworth2012optimal}, recall that $\Psi(x)=d(\pi_1P_1(x)-\pi_2P_2(x))$,  then
	\begin{eqnarray*}
		&&\int_{\mathbb{R}^d} \left( P\left( \sum_{i=1}^k W_iY_i\leq \frac{1}{2} \right)-1_{\{ \eta(x)<1/2 \}} \right)d\bar{P}(x) \\&=& \{1+o(1) \} \int_{\mathcal{S}} \int_{-\epsilon_{k,n}}^{\epsilon_{k,n}} t\|\dot{\Psi}(x_0)\| \left( P({S}_{k,n}(x_0^t)<1/2) -1_{\{t<0 \}} \right)dtd\text{Vol}^{d-1}(x_0)+r_1,
	\end{eqnarray*}
	where
	\begin{eqnarray*}
		r_1&=&\int_{\mathbb{R}^d\backslash \mathcal{S}^{\epsilon_{k,n}}} \left( P\left( \sum_{i=1}^k W_iY_i\leq \frac{1}{2}  \right)-1_{\{ \eta(x)<1/2 \}} \right)d\bar{P}(x)\\
		&=&\int_{\mathbb{R}^d\backslash \mathcal{S}^{\epsilon_{k,n}}} \mathbb{E}_{R_{k+1}}\left( P\left( \sum_{i=1}^k Z_i(x,R_{k+1})<0\right)-1_{\{ \eta(x)<1/2 \}} \right)d\bar{P}(x).
	\end{eqnarray*}
	For $r_1$,
	\begin{eqnarray*}
		0&\geq &\int_{\mathbb{R}^d\backslash\mathcal{S}^{\epsilon_{k,n}}\cap \{x|\eta(x)<1/2  \}} \mathbb{E}_{R_{k+1}}\left( P\left( \sum_{i=1}^k Z_i(x,R_{k+1})\leq 0\right)-1_{\{ \eta(x)<1/2 \}}  \right)d\bar{P}(x)\\
		&=& -\int_{\mathbb{R}^d\backslash\mathcal{S}^{\epsilon_{k,n}}\cap \{x|\eta(x)<1/2  \}}\\&&\qquad\qquad\qquad\mathbb{E}_{R_{k+1}}\left( P\left( \sum_{i=1}^k Z_i(x,R_{k+1})-k\mathbb{E}Z_1(x,R_{k+1})> -k\mathbb{E}Z_1(x,R_{k+1})\right)\right)d\bar{P}(x).
	\end{eqnarray*}
	Using non-uniform Berry-Essen Theorem, when $\mathbb{E}Z_1^3(x,R_{k+1})<\infty$, i.e.,$\gamma<d/3$, it becomes
	\begin{eqnarray*}
		&&-\int_{\mathbb{R}^d\backslash\mathcal{S}^{\epsilon_{k,n}}\cap \{x|\eta(x)<1/2  \}}\\&&\qquad\qquad\qquad\mathbb{E}_{R_{k+1}}\left( P\left( \sum_{i=1}^k Z_i(x,R_{k+1})-k\mathbb{E}Z_1(x,R_{k+1})> -k\mathbb{E}Z_1(x,R_{k+1})\right)\right)d\bar{P}(x)\\
		&\leq&\int_{\mathbb{R}^d\backslash\mathcal{S}^{\epsilon_{k,n}}\cap \{x|\eta(x)<1/2  \}}  \mathbb{E}_{R_{k+1}}\bar{\Phi}\left( -\frac{\sqrt{k}\mathbb{E}Z_1(x,R_{k+1})}{Var(Z_1(x,R_{k+1})) } \right)d\bar{P}(x)\\
		&&+c_1\frac{1}{\sqrt{k}}\int_{\mathbb{R}^d\backslash\mathcal{S}^{\epsilon_{k,n}}\cap \{x|\eta(x)<1/2  \}}\mathbb{E}_{R_{k+1}} \frac{ 1 }{ 1+k^{3/2}|\mathbb{E}Z_1(x,R_{k+1})|^3 }d\bar{P}(x),
	\end{eqnarray*}
	
	where $\bar{\Phi}(x)=1-\Phi(x)$. Since $s_{k,n,\gamma}(x,R_{k+1})=o(\epsilon_{k,n})$, $r_1=o(s_{k,n,\gamma}^2(x,R_{k+1}))$.
	
	By the definition of $\epsilon_{k,n}$, we have
	\begin{eqnarray*}
		\exp(-\epsilon_{k,n}^2/s_{k,n,\gamma}^2(x,R_{k+1}))&=&o(s_{k,n,\gamma}^2)+o(1/k),\\
		\inf\limits_{x\in\mathbb{R}^d\backslash S^{\epsilon_{k,n}}}|\eta(x)-1/2|&\geq& c_3\epsilon_{k,n}.
	\end{eqnarray*}
	As a result, using Berstain inequality, $r_1$ is a smaller order term compared with $s_{k,n,\gamma}^2$ when $s_{k,n,\gamma}^2(R_{k+1})=o(1)$, hence $r_1=o(s_{k,n,\gamma}^2)$.

	\textit{Step 3}: now we apply non-uniform Berry-Esseen Theorem. From \textit{Step 3}, we have
	\begin{eqnarray*}
		&&\int_{\mathcal{S}}\int_{-\epsilon_{k,n}}^{\epsilon_{k,n}} \mathbb{E}_{R_{k+1}}t\|\dot{\Psi}(x_0)\| \left( P({S}_{k,n}(x_0^t)<1/2|R_{k+1}) -1_{\{t<0 \}} \right)dtd\text{Vol}^{d-1}(x_0)\\&=& \int_{\mathcal{S}} \int_{-\epsilon_{k,n}}^{\epsilon_{k,n}}\mathbb{E}_{R_{k+1}} t\|\dot{\Psi}(x_0)\|\left( \Phi\left( \frac{ -k\mathbb{E}Z_1(x_0^t,R_{k+1})}{\sqrt{kVar(Z_1(x_0^t,R_{k+1}))}}\right)-1_{\{t<0 \}} \right)dtd\text{Vol}^{d-1}(x_0) + r_2,
	\end{eqnarray*}
	where based on non-uniform Berry-Esseen Theorem:
	\begin{eqnarray*}
		&&\bigg|P\left( \frac{\sum_{i=1}^k Z_i(x_0^t,R_{k+1})-k\mathbb{E}Z_1(x_0^t,R_{k+1})}{\sqrt{kVar(Z_1(x_0^t,R_{k+1}))}}< \frac{ -k\mathbb{E}Z_1(x_0^t,R_{k+1})}{\sqrt{kVar(Z_1(x_0^t,R_{k+1}))}}\right)\\&&\qquad-  \Phi\left( \frac{ -k\mathbb{E}Z_1(x_0^t,R_{k+1})}{\sqrt{kVar(Z_1(x_0^t,R_{k+1}))}}\right) \bigg|\\&\leq& c\frac{k\mathbb{E}|Z_1(x_0^t,R_{k+1})|^3}{k^{3/2}Var^{3/2}(Z_1(x_0^t,R_{k+1}))}\frac{1}{    1+ \left|\frac{ -k\mathbb{E}Z_1(x_0^t,R_{k+1})}{\sqrt{kVar(Z_1(x_0^t,R_{k+1}))}}\right|^3
		},
	\end{eqnarray*} and
	
	\begin{eqnarray*}
		|r_2|&\leq& \int_{\mathcal{S}} \int_{-\epsilon_{k,n}}^{\epsilon_{k,n}}\\&&\qquad \mathbb{E}_{R_{k+1}}t\|\dot{\Psi}(x_0)\| \frac{k\mathbb{E}|Z_1(x_0^t,R_{k+1})|^3}{k^{3/2}Var^{3/2}(Z_1(x_0^t,R_{k+1}))}\frac{1}{    1+ \left|\frac{ -k\mathbb{E}Z_1(x_0^t,R_{k+1})}{\sqrt{kVar(Z_1(x_0^t,R_{k+1}))}}\right|^3
		}dtd\text{Vol}^{d-1}(x_0).
	\end{eqnarray*}
	For $r_2$,
	\begin{eqnarray*}
		r_2&\leq&  \int_{\mathcal{S}} \int_{-\epsilon_{k,n}}^{\epsilon_{k,n}}\\&&\qquad\mathbb{E}_{R_{k+1}} t\|\dot{\Psi}(x_0)\| \frac{k\mathbb{E}|Z_1(x_0^t,R_{k+1})|^3}{k^{3/2}Var^{3/2}(Z_1(x_0^t,R_{k+1}))}\frac{1}{    1+ \left|\frac{ -k\mathbb{E}Z_1(x_0^t,R_{k+1})}{\sqrt{kVar(Z_1(x_0^t,R_{k+1}))}}\right|^3
		}dtd\text{Vol}^{d-1}(x_0)\\
		&=& \frac{c_1}{\sqrt{k}}\int_{\mathcal{S}} \int_{-\epsilon_{k,n}}^{\epsilon_{k,n}} \mathbb{E}_{R_{k+1}}t\|\dot{\Psi}(x_0)\| \frac{1}{    1+ \left|\frac{ -k\mathbb{E}Z_1(x_0^t,R_{k+1})}{\sqrt{kVar(Z_1(x_0^t,R_{k+1}))}}\right|^3
		}dtd\text{Vol}^{d-1}(x_0)\\
		&=&\frac{c_2}{\sqrt{k}}\int_{\mathcal{S}} \int_{|t|<s_{k,n,\gamma}(x)} t\|\dot{\Psi}(x_0)\| dtd\text{Vol}^{d-1}(x_0)\\
		&&+\frac{c_3}{\sqrt{k}}\int_{\mathcal{S}} \int_{s_{k,n,\gamma}(x)<|t|<\epsilon_{k,n}}t\|\dot{\Psi}(x_0)\| \frac{1}{    1+ k^{3/2} t^3
		}dtd\text{Vol}^{d-1}(x_0)\\
		&=& o(s_{k,n,\gamma}^2).
	\end{eqnarray*}
	\textit{Step 4}: the integral becomes
	\begin{eqnarray*}
		&&\int_{\mathcal{S}}\mathbb{E}_{R_{k+1}}\int_{-\epsilon_{k,n}}^{\epsilon_{k,n}} t\|\dot{\Psi}(x_0)\|\left(\Phi\left( \frac{ -k\mathbb{E}Z_1(x_0^t,R_{k+1})}{\sqrt{kVar(Z_1(x_0^t,R_{k+1}))}}\right)-1_{\{t<0\}}\right)dtd\text{Vol}^{d-1}(x_0)\\
		&=&\int_{\mathcal{S}}\mathbb{E}_{R_{k+1}}\int_{-\epsilon_{k,n}}^{\epsilon_{k,n}} t\|\dot{\Psi}(x_0)\|\\&&\qquad\left(\Phi\left(-\frac{t\|\dot{\eta}(x_0)\|}{\sqrt{s_{k,n,\gamma}^2(x,R_{k+1})/k}} -\frac{ \mathbb{E}(R_1/R_{k+1})^{-\gamma}a(x_0^t)R_1^2}{\sqrt{s_{k,n,\gamma}^2(x,R_{k+1})/k}}\right)-1_{\{t<0\}}\right)dtd\text{Vol}^{d-1}(x_0)\\&&+r_3+o\\
		&=&\int_{\mathcal{S}}\mathbb{E}_{R_{k+1}}\int_{\mathbb{R}} t\|\dot{\Psi}(x_0)\|\\&&\qquad\left(\Phi\left(-\frac{t\|\dot{\eta}(x_0)\|}{\sqrt{s_{k,n,\gamma}^2(x,R_{k+1})/k}} -\frac{ \mathbb{E}(R_1/R_{k+1})^{-\gamma}a(x_0^t)R_1^2}{\sqrt{s_{k,n,\gamma}^2(x,R_{k+1})/k}}\right)-1_{\{t<0\}}\right)dtd\text{Vol}^{d-1}(x_0)\\&&+r_3+r_4+o\\
		&=&\int_{\mathcal{S}}\int_{\mathbb{R}} t\|\dot{\Psi}(x_0)\|\left(\Phi\left(-\frac{t\|\dot{\eta}(x_0)\|}{\sqrt{s_{k,n,\gamma}^2(x)/k}} -\frac{ \mathbb{E}(R_1/R_{k+1})^{-\gamma}a(x_0^t)R_1^2}{\sqrt{s_{k,n,\gamma}^2(x)/k}}\right)-1_{\{t<0\}}\right)dtd\text{Vol}^{d-1}(x_0)\\&&+r_3+r_4+r_5+o\\
		&=& \frac{B_1}{4k}\frac{ \mathbb{E}(R_1/R_{k+1})^{-2\gamma} }{\mathbb{E}^2(R_1/R_{k+1})^{-\gamma}} + \int_{S}\frac{f(x_0)}{\|\dot{\eta}(x_0)\|}a^2(x_0) \frac{  \mathbb{E}^2(R_1/R_{k+1})^{-\gamma}R_1^{2}}{\mathbb{E}^2(R_1/R_{k+1})^{-\gamma}} d\text{Vol}^{d-1}(x_0) +r_3+r_4+r_5+o.
	\end{eqnarray*}
	Note that $\|\dot{\Psi}(x_0)\|/\|\dot{\eta}(x_0)\|=2f(x_0)$. The last step follows Proposition \ref{S.1} and the fact that $R_{k+1}$ does not affect the dominant parts. The term $\mathbb{E}((R_1/R_{k+1})^{-\gamma}|R_{k+1})$ is almost the same for all $R_{k+1}$. For the small order terms, following \cite{samworth2012optimal} we obtain
	\begin{eqnarray*}
		r_3&=&\int_{\mathcal{S}}\mathbb{E}_{R_{k+1}} \int_{-\epsilon_{k,n}}^{\epsilon_{k,n}} t\|\dot{\Psi}(x_0)\|\bigg(\Phi\left( \frac{ -k\mathbb{E}Z_1(x_0^t,R_{k+1})}{\sqrt{kVar(Z_1(x_0^t,R_{k+1}))}}\right)\\&&\qquad\qquad-\Phi\left(-\frac{t\|\dot{\eta}(x_0)\|}{\sqrt{s_{k,n,\gamma}^2(x,R_{k+1})/k}} -\frac{ \mathbb{E}(R_1/R_{k+1})^{-\gamma}a(x_0^t)R_1^2}{\sqrt{s_{k,n,\gamma}^2(x,R_{k+1})/k}}\right) \bigg)dtd\text{Vol}^{d-1}(x_0),\\
		&=& o(s_{k,n,\gamma}^2+t_{k,n,\gamma}^2),
	\end{eqnarray*}
	and
	\begin{eqnarray*}
		r_4&=& \int_{\mathcal{S}}\frac{\|\dot{\Psi}(x_0)\|}{\|\dot{\eta}(x_0)\|^2} \mathbb{E}_{R_{k+1}}\int_{\mathbb{R}\backslash [-\epsilon_{k,n},\epsilon_{k,n}]}\\&&\qquad t\|\dot{\Psi}(x_0)\|\left( \Phi\left(-\frac{t\|\dot{\eta}(x_0)\|}{\sqrt{s_{k,n,\gamma}^2/k}} -\frac{ \mathbb{E}(R_1/R_{k+1})^{-\gamma}a(x_0^t)R_1^2}{\sqrt{s_{k,n,\gamma}^2/k}}\right)-1_{\{v<0 \}} \right)dtd\text{Vol}^{d-1}(x_0)\\&=&o(s_{k,n,\gamma}^2).
	\end{eqnarray*}
	The term $r_5$ is the difference between the normal probability given $R_{k+1}$ and the one after taking expectation. Similar with \cite{cannings2017local}, for each $x$, when $|t|<\epsilon_{k,n}$, we have
	\begin{eqnarray*}
		&&\mathbb{E}_{R_{k+1}}\Phi\left(-\frac{t\|\dot{\eta}(x_0)\|}{\sqrt{s_{k,n,\gamma}^2(x,R_{k+1})/k}} -\frac{ \mathbb{E}(R_1/R_{k+1})^{-\gamma}a(x_0^t)R_1^2}{\sqrt{s_{k,n,\gamma}^2(x,R_{k+1})/k}}\right)\\
		&=& \Phi\left(-\frac{t\|\dot{\eta}(x_0)\|}{\sqrt{s_{k,n,\gamma}^2(x,R_{k+1})/k}} -\frac{ \mathbb{E}(R_1/R_{k+1})^{-\gamma}a(x_0^t)R_1^2}{\sqrt{s_{k,n,\gamma}^2(x,R_{k+1})/k}}\right)+O\left( k Var(t_{k,n,\gamma}(x,R_{k+1})) \right)+o.
	\end{eqnarray*}
	Following step 3 in \cite{cannings2017local}, we obtain
	\begin{eqnarray*}
		Var(t_{k,n,\gamma}(x,R_{k+1}))\leq \frac{1}{k^2}\sum_{j=1}^k \mathbb{E}(\eta(X^1)-\eta(x))^2=O\left(\frac{1}{k} r_{2p}^2\right).
	\end{eqnarray*}
	For the case when $|t|\gg s_{k,n,\gamma}+t_{k,n,\gamma}$, differentiate normal cdf twice still leads to very small probability, thus for each $x_0$, we have
	\begin{eqnarray*}
		&&\int t\mathbb{E}_{R_{k+1}}\Phi\left(-\frac{t\|\dot{\eta}(x_0)\|}{\sqrt{s_{k,n,\gamma}^2(x,R_{k+1})/k}} -\frac{ \mathbb{E}(R_1/R_{k+1})^{-\gamma}a(x_0^t)R_1^2}{\sqrt{s_{k,n,\gamma}^2(x,R_{k+1})/k}}\right)dt\\
		&=&\int t\Phi\left(-\frac{t\|\dot{\eta}(x_0)\|}{\sqrt{s_{k,n,\gamma}^2(x,R_{k+1})/k}} -\frac{ \mathbb{E}(R_1/R_{k+1})^{-\gamma}a(x_0^t)R_1^2}{\sqrt{s_{k,n,\gamma}^2(x,R_{k+1})/k}}\right)dt+o(s_{k,n,\gamma}^2+t_{k,n,\gamma}^2).
	\end{eqnarray*}

	\subsection{ Connecting Multiplicative Constants w.r.t $\gamma$}\label{sec:cor:proof}
	
	To show Theorem \ref{thm:main}, we need to work out the multiplicative constants.

	For classification, given $x$, we know that if $X$ follows multi-dimensional uniform distribution with density $1/f(x)$, for some constant $c_d$ that only depends on $d$,
	\begin{eqnarray*}
		\mathbb{E}({R}_1/{R}_{k+1})^{-2\gamma} 
		&=& \frac{d}{d-2\gamma}+o,\\
		\mathbb{E}({R}_1/{R}_{k+1})^{-\gamma} &=& \frac{d}{d-\gamma}+o,\\
		\mathbb{E}({R}_1/{R}_{k+1})^{-\gamma}R_1^{2}&=& \mathbb{E}({R}_1/{R}_{k+1})^{2-\gamma}R_{k+1}^2 =c_d\left(\frac{k}{nf(x)}\right)^{\frac{2}{d}}\frac{d}{d+2-\gamma}+o.
	\end{eqnarray*}

	For regression, one more step needed compared with classification is to evaluate 
	\begin{eqnarray*}
		k\mathbb{E}\left[\frac{(R_{1}/R_{k+1})^{-2\gamma}}{ (\sum_{i=1}^k (R_{i}/R_{k+1})^{-\gamma})^2}\right]\mathbb{E}\sigma(X)^2 + k^2 \mathbb{E}\left(a^2(X)\mathbb{E}^2\left[\frac{R_1^2(R_1/R_{k+1})^{-\gamma}}{\sum_{i=1}^k(R_i/R_{k+1})^{-\gamma}}\right]\right).
	\end{eqnarray*}
	
	The sum of ratios $\sum_{i=1}^k (R_i/R_{k+1})^{-\gamma}$ is hard to evaluated directly in the denominator, hence we use upper bound and lower bound on it. Since $d-3\gamma>0$, using non-uniform Berry-Essen Theorem, given $R_{k+1}$, we have
	\begin{eqnarray*}
		&&\left|P\left(\sum_{i=1}^k (R_i/R_{k+1})^{-\gamma} -k\mathbb{E}(R_1/R_{k+1})^{-\gamma} > \xi  \right)- \bar{\Phi}\left( \frac{ \xi }{\sqrt{kVar((R_1/R_{k+1})^{-\gamma})}}\bigg|R_{k+1} \right)\right| \\&\leq& \frac{c}{1+(\xi/\sqrt{k})^3}.
	\end{eqnarray*}
	
	Therefore, taking $\xi=\delta_kk\mathbb{E}(R_1/R_{k+1})^{-\gamma}$, 
	\begin{eqnarray*}
		&&P\left(\sum_{i=1}^k (R_i/R_{k+1})^{-\gamma}  > (\delta_k+1)k\mathbb{E}(R_1/R_{k+1})^{-\gamma}\bigg|R_{k+1} \right)\\& \leq&
		\bar{\Phi}\left( \frac{ \delta_k\sqrt{k}\mathbb{E}(R_1/R_{k+1})^{-\gamma  }}{\sqrt{Var((R_1/R_{k+1})^{-\gamma})}} \bigg|R_{k+1}\right) +\frac{c}{1+(\sqrt{k}\delta_k)^3}.
	\end{eqnarray*}
	
	Note that $(R_1/R_{k+1})^{-\gamma}$ is always larger than 1, hence
	
	\begin{eqnarray*}
		&&\mathbb{E}\left[\frac{(R_{1}/R_{k+1})^{-2\gamma}}{ (\sum_{i=1}^k (R_{i}/R_{k+1})^{-\gamma})^2}\right]\\&\leq&  \mathbb{E}_{R_{k+1}}\left[\frac{\mathbb{E}(R_1/R_{k+1})^{-2\gamma}}{ (1-\delta_k)^2k^2\mathbb{E}^2(R_1/R_{k+1})^{-\gamma} }\right]\\
		&&+ \mathbb{E}_{R_{k+1}}\left[P\left(\sum_{i=1}^k (R_i/R_{k+1})^{-\gamma}  < (1-\delta_k)k\mathbb{E}(R_1/R_{k+1})^{-\gamma} \bigg|R_{k+1}\right)\frac{\mathbb{E}(R_1/R_{k+1})^{-2\gamma}}{k^2}\right]+o\\
		&\leq& \frac{1}{k^2(1-\delta_k)^2} \frac{\mathbb{E}(R_1/R_{k+1})^{-2\gamma}}{\mathbb{E}^2(R_1/R_{k+1})^{-\gamma}} +\frac{\mathbb{E}(R_1/R_{k+1})^{-2\gamma}}{k^2}	\bar{\Phi}\left( \frac{ \delta_k\sqrt{k}\mathbb{E}(R_1/R_{k+1})^{-\gamma  }}{\sqrt{Var((R_1/R_{k+1})^{-\gamma})}} \right) \\&&+\frac{\mathbb{E}(R_1/R_{k+1})^{-2\gamma}}{k^2}	\frac{c}{1+(\sqrt{k}\delta_k)^3}+o,
	\end{eqnarray*}
	while
	\begin{eqnarray*}
		\mathbb{E}\left[\frac{(R_{1}/R_{k+1})^{-2\gamma}}{ (\sum_{i=1}^k (R_{i}/R_{k+1})^{-\gamma})^2}\right]&\geq& \frac{1}{k^2(1+\delta_k)^2} \frac{\mathbb{E}(R_1/R_{k+1})^{-2\gamma}}{\mathbb{E}^2(R_1/R_{k+1})^{-\gamma}}-\frac{\mathbb{E}(R_1/R_{k+1})^{-2\gamma}}{k^2}	\frac{c}{1+(\sqrt{k}\delta_k)^3}\\&&+o.
	\end{eqnarray*}
	Hence taking $\delta_k$ such that $\delta_k\rightarrow0$ while $\delta_k\sqrt{k}\rightarrow\infty$, we have
	\begin{eqnarray*}
		\mathbb{E}\left[\frac{(R_{1}/R_{k+1})^{-2\gamma}}{ (\sum_{i=1}^k (R_{i}/R_{k+1})^{-\gamma})^2}\right]=\frac{1}{k^2}\frac{(d-\gamma)^2}{d(d-2\gamma)}+o,
	\end{eqnarray*}
	and similarly
	\begin{eqnarray*}
		\mathbb{E}\left[\frac{R_1^2(R_1/R_{k+1})^{-\gamma}}{\sum_{i=1}^k(R_i/R_{k+1})^{-\gamma}}\right]= \frac{1}{k} \left(\frac{k}{nf(x)}\right)^{\frac{2}{d}}\frac{d-\gamma}{d+2-\gamma}+o.
	\end{eqnarray*}

	\section{Proof of Theorem \ref{CIS}}\label{sec:cis:proof}
	
	The proof is similar with Theorem 1 in \cite{sun2016stabilized}.
	
	From the definition of CIS, we have
	\begin{eqnarray*}
		CIS(\gamma)/2&=&\int_{\mathcal{R}} P(S_{k,n,\gamma}(x)\geq 1/2)\left(1-P(S_{k,n,\gamma}(x)\geq 1/2) \right)dP(x)\\
		&=& \int_{\mathcal{R}} \left(P(S_{k,n,\gamma}(x)\geq 1/2)-1_{\{ \eta(x)\leq 1/2 \}}\right)d{P}(x)\\
		&&-\int_{\mathcal{R}} \left(P^2(S_{k,n,\gamma}(x)\geq 1/2)-1_{\{ \eta(x)\leq 1/2 \}}\right)d{P}(x).
	\end{eqnarray*}

	Based on the definition of $Z_i(x,R_{k+1})$ in (\ref{eqn:Z}), the derivation of $\mu_{k,n,\gamma}(x,R_{k+1})$ and $s_{k,n,\gamma}(x,R_{k+1})$, follow the same procedures as in Theorem \ref{thm:main}, we obtain
	\begin{eqnarray*}
		&&\int_{\mathcal{R}} \left(P(S_{k,n,\gamma}(x)\geq 1/2)-1_{\{ \eta(x)\leq 1/2 \}}\right)d{P}(x)\\&=&
		\int_{S}\int_{-\epsilon_{k,n}}^{\epsilon_{k,n}} f(x_0^t)\left\{   \mathbb{E}_{R_{k+1}} P\left( S_{k,n,\gamma}(x_0^t)<1/2 |R_{k+1}\right)-1_{ \{t<0 \} } \right\}dtd\text{Vol}^{d-1}(x_0)+o\\
		&=& \int_{S}\mathbb{E}_{R_{k+1}}\int_{\mathbb{R}} f(x_0) \left\{  \Phi\left(-\frac{t\|\dot{\eta}(x_0)\|}{\sqrt{s_{k,n,\gamma}^2/k}} -\frac{ \mathbb{E}(R_1/R_{k+1})^{-\gamma}a(x_0^t)R_1^2}{\sqrt{s_{k,n,\gamma}^2/k}}\right)-1_{\{ t<0 \}} \right\}dtd\text{Vol}^{d-1}(x_0)\\&&+o,
	\end{eqnarray*}
	and similarly,
	\begin{eqnarray*}
		&&\int_{\mathcal{R}} \left(P^2(S_{k,n,\gamma}\geq 1/2|R)-1_{\{ \eta(x)\leq 1/2 \}}\right)d\bar{P}(x)\\&=&
		\int_{S}\int_{-\epsilon_{k,n}}^{\epsilon_{k,n}} f(x_0^t)\left\{    P^2\left( S_{k,n,\gamma}(x_0^t)<1/2 \right)-1_{ \{t<0 \} } \right\}dtd\text{Vol}^{d-1}(x_0)+o\\
		&=& \int_{S}\mathbb{E}_{R_{k+1}}\int_{\mathbb{R}}\\&&\qquad f(x_0)  \left\{  \Phi^2\left(-\frac{t\|\dot{\eta}(x_0)\|}{\sqrt{s_{k,n,\gamma}^2/k}} -\frac{ \mathbb{E}(R_1/R_{k+1})^{-\gamma}a(x_0^t)R_1^2}{\sqrt{s_{k,n,\gamma}^2/k}}\right)-1_{\{ t<0 \}} \right\}dtd\text{Vol}^{d-1}(x_0)+o.
	\end{eqnarray*}
	Adopting Proposition \ref{S.1} and the fact that $s_{k,n,\gamma}(x,R_{k+1})$ is little changed by $x$ and $R_{k+1}$, treating $\Phi$ and $\Phi^2$ as two distribution functions, we have
	\begin{eqnarray*}
		CIS(\gamma) = \frac{B_1}{\sqrt{\pi}}\frac{1}{\sqrt{k}} \mathbb{E}\sqrt{s_{k,n,\gamma}^2(X)}+o=\frac{B_1}{\sqrt{\pi}}\frac{1}{\sqrt{k}} \sqrt{1+\frac{\gamma^2}{d(d-2\gamma)}}+o.
	\end{eqnarray*}
	
	\section{Regret under Testing Data Corruption}
	This section is the proof for Regret under testing data corruption.
	\begin{proof}
		\begin{eqnarray*}
			&&\int_{\mathcal{S}} \int_{-\epsilon_{k,n,\omega}}^{\epsilon_{k,n,\omega}} t\|\dot{\Psi}(x_0)\|\left( \Phi\left( \frac{ k\mathbb{E}(1/2-Y_1)}{\sqrt{kVar(Y_1)}}\right)-1_{\{t<0 \}} \right)dtd\text{Vol}^{d-1}(x_0)\\
			&=&	\int_{\mathcal{S}}\int_{\mathbb{R}} t\|\dot{\Psi}(x_0)\|\\&&\quad\left(\Phi\left(-\frac{t\|\dot{\eta}(x_0)\|-sign(t)\omega\|\dot{\eta}(x_0)\|}{\sqrt{s_{k,n}^2}} -\frac{ b(x_0)t_{k,n}(x_0))}{\sqrt{s_{k,n}^2}}\right)-1_{\{t<0\}}\right)dtd\text{Vol}^{d-1}(x_0)+o\\
			&=& \frac{1}{2}\int_{\mathcal{S}}\int_{\mathbb{R}} t\|\dot{\Psi}(x_0)\|\left(\Phi\left(-\frac{t\|\dot{\eta}(x_0)\|+\omega\|\dot{\eta}(x_0)\|}{\sqrt{s_{k,n}^2}} -\frac{ b(x_0)t_{k,n}(x_0))}{\sqrt{s_{k,n}^2}}\right)-1_{\{t<0\}}\right)dtd\text{Vol}^{d-1}(x_0)\\
			&&+\frac{1}{2}\int_{\mathcal{S}}\int_{\mathbb{R}} t\|\dot{\Psi}(x_0)\|\\&&\qquad\qquad\qquad\left(\Phi\left(-\frac{t\|\dot{\eta}(x_0)\|-\omega\|\dot{\eta}(x_0)\|}{\sqrt{s_{k,n}^2}} -\frac{ b(x_0)t_{k,n}(x_0))}{\sqrt{s_{k,n}^2}}\right)-1_{\{t<0\}}\right)dtd\text{Vol}^{d-1}(x_0)\\&&+r_5+o\\
			&=&\frac{B_1}{4k}+\frac{1}{2}\int_{S}\frac{\|\dot{\Psi}(x_0)\|}{\|\dot{\eta}(x_0)\|^2} \left(b(x_0)^2\mathbb{E}^2R_1(x)^2+\omega^2\|\dot{\eta}(x_0)\|^2\right) d\text{Vol}^{d-1}(x_0)+ r_5+o.
		\end{eqnarray*}
		\begin{eqnarray*}
			2r_5&=&   \int_{\mathcal{S}}\int_{0}^{2\omega} (t-2\omega)\|\dot{\Psi}(x_0)\|\Phi\left(-\frac{t\|\dot{\eta}(x_0)\|-\omega\|\dot{\eta}(x_0)\|}{\sqrt{s_{k,n}^2}} -\frac{ b(x_0)t_{k,n}(x_0)}{\sqrt{s_{k,n}^2}}\right)dtd\text{Vol}^{d-1}(x_0)\\
			&&+ \int_{\mathcal{S}}\int_{-2\omega}^0 (t+2\omega)\|\dot{\Psi}(x_0)\|\Phi\left(-\frac{t\|\dot{\eta}(x_0)\|+\omega\|\dot{\eta}(x_0)\|}{\sqrt{s_{k,n}^2}} -\frac{ b(x_0)t_{k,n}(x_0)}{\sqrt{s_{k,n}^2}}\right)dtd\text{Vol}^{d-1}(x_0)\\&&+o\\
			&=& 2\int_{\mathcal{S}}\int_{-\omega}^\omega t\|\dot{\Psi}(x_0)\|\Phi\left(-\frac{t\|\dot{\eta}(x_0)\|}{\sqrt{s_{k,n}^2}} -\frac{ b(x_0)t_{k,n}(x_0)}{\sqrt{s_{k,n}^2}}\right)dtd\text{Vol}^{d-1}(x_0)+o\\
			&=& 2\int_{\mathcal{S}}\int_{0}^\omega t\|\dot{\Psi}(x_0)\|\Phi\left(-\frac{t\|\dot{\eta}(x_0)\|}{\sqrt{s_{k,n}^2}} -\frac{ b(x_0)t_{k,n}(x_0)}{\sqrt{s_{k,n}^2}}\right)dtd\text{Vol}^{d-1}(x_0)\\
			&&+2\int_{\mathcal{S}}\int_{-\omega}^0 t\|\dot{\Psi}(x_0)\|\Phi\left(-\frac{t\|\dot{\eta}(x_0)\|}{\sqrt{s_{k,n}^2}} -\frac{ b(x_0)t_{k,n}(x_0)}{\sqrt{s_{k,n}^2}}\right)dtd\text{Vol}^{d-1}(x_0)+o,
		\end{eqnarray*}
		\blue{and some further calculation reveals that}
		\begin{eqnarray*}
			2r_5&=& -2\int_{\mathcal{S}}\int_{-\omega}^0 t\|\dot{\Psi}(x_0)\|\Phi\left(\frac{t\|\dot{\eta}(x_0)\|}{\sqrt{s_{k,n}^2}} -\frac{ b(x_0)t_{k,n}(x_0)}{\sqrt{s_{k,n}^2}}\right)dtd\text{Vol}^{d-1}(x_0)\\
			&&+2\int_{\mathcal{S}}\int_{-\omega}^0 t\|\dot{\Psi}(x_0)\|\Phi\left(-\frac{t\|\dot{\eta}(x_0)\|}{\sqrt{s_{k,n}^2}} -\frac{ b(x_0)t_{k,n}(x_0)}{\sqrt{s_{k,n}^2}}\right)dtd\text{Vol}^{d-1}(x_0)+o\\
			&=& 2\int_{\mathcal{S}}\int_{-\omega}^0 t\|\dot{\Psi}(x_0)\|\bigg[\Phi\left(-\frac{t\|\dot{\eta}(x_0)\|}{\sqrt{s_{k,n}^2}} -\frac{ b(x_0)t_{k,n}(x_0)}{\sqrt{s_{k,n}^2}}\right)\\&&\qquad\qquad+\Phi\left(-\frac{t\|\dot{\eta}(x_0)\|}{\sqrt{s_{k,n}^2}} +\frac{ b(x_0)t_{k,n}(x_0)}{\sqrt{s_{k,n}^2}}\right)\bigg]dtd\text{Vol}^{d-1}(x_0)\\
			&& +\omega^2\int_{\mathcal{S}}\|\dot{\Psi}(x_0)\|d\text{Vol}^{d-1}(x_0)+o.
		\end{eqnarray*}
		Taking gradient on $r_5$ w.r.t. $t_{k,n}$, the gradient is positive. As a result, when introducing interpolation and fixing $k$, $r_5$ becomes smaller.
	\end{proof}
	
	\section{Formal Representations for Corollary \ref{coro:corruption}}\label{sec:corruption}
	This section is a formal representations for Corollary \ref{coro:corruption}.
	\begin{corollary}
		Under the conditions stated in Corollary \ref{coro:corruption}, for
		\begin{itemize}
			\item random perturbation: if $\omega^3=o(n^{-4/(d+3)})$, the following result holds:
			\begin{eqnarray*}
				\text{Regret}(k,n,\gamma)&=&\underbrace{ \frac{(d-\gamma)^2}{d(d-2\gamma)}\frac{1}{4k}\int_{S}\frac{f(x_0)}{\|\dot{\eta}(x_0)\|}d\text{Vol}^{d-1}(x_0)}_{Variance}\\&&+\underbrace{\frac{(d-\gamma)^2}{(d+2-\gamma)^2}\frac{(d+2)^2}{d^2}\int_{S}\frac{f(x_0)a(x_0)^2}{\|\dot{\eta}(x_0)\|} \mathbb{E}(R_1^2|X=x_0) d\text{Vol}^{d-1}(x_0)}_{Bias}\\
				&&+ \underbrace{\frac{1}{2}\int_{\mathcal{S}}\frac{\omega^2}{d}\|\dot{\Psi}(x_0)\| d\text{Vol}^{d-1}(x_0)}_{Corruption}+Remainder.
			\end{eqnarray*}
			The corruption is not related to $\gamma$.
			\item black-box attack: for simplicity, we use $\eta$ instead of $\widetilde\eta$. In this case, if $\omega^3=o(n^{-4/(d+3)})$, the regret becomes
			\begin{eqnarray*}
				\text{Regret}(k,n,\gamma)&=&\underbrace{ \frac{(d-\gamma)^2}{d(d-2\gamma)}\frac{1}{4k}\int_{S}\frac{f(x_0)}{\|\dot{\eta}(x_0)\|}d\text{Vol}^{d-1}(x_0)}_{Variance}\\&&+\underbrace{\frac{(d-\gamma)^2}{(d+2-\gamma)^2}\frac{(d+2)^2}{d^2}\int_{S}\frac{f(x_0)a(x_0)^2}{\|\dot{\eta}(x_0)\|} \mathbb{E}(R_1^2|X=x_0) d\text{Vol}^{d-1}(x_0)}_{Bias}\\
				&&+Corruption+Remainder,
			\end{eqnarray*}
			where Corruption decreases when $\gamma$ slightly increases from zero.
		\end{itemize}
	\end{corollary}
	
	\section{Variance-Bias Trade-off in General Weighting Schemes}\label{sec:ownn}
	This section discusses the variance-bias trade-off in general weighting schemes.
	
	Besides OWNN and interpolated-NN, we found that the benefit from the variance-bias trade-off exists for a general class of weights (not essential to be interpolated).  Similar as for interpolated-NN, when $\gamma$ is closed to zero, the increase of variance is approximately a quadratic function in $\gamma$, while bias is linearly reduced. 
	\begin{corollary}\label{coro:general}
		Denote $x$ as $R_{1}/R_{k+1}$ and  $\phi(x,\gamma):[0,1]\times [0,\infty)\rightarrow [0,\infty)$ as a function such that $\phi(x,0)\equiv1$, and taking $\phi'(x,\gamma)=\partial \phi/\partial \gamma$. If
		\begin{eqnarray*}
			\Delta:=\left(\int_0^1\phi'(x,0)x^{d+1}dx\right)\left(\int_0^1x^{d-1}dx\right)-\left(\int_0^1\phi'(x,0)x^{d-1}dx\right)\left(\int_0^1x^{d+1}dx\right)<0,
		\end{eqnarray*}
		then when sightly increasing $\gamma$ (to cause interpolation / allocating more weight on closer neighbors), it is guaranteed that the overall MSE / Regret will get decreased. 
	\end{corollary}

	\begin{proof}[Proof of Corollary \ref{coro:general}]
		
		When $\gamma$ is chosen that $\int_0^1 \phi(x,\gamma)^3 x^{d-1} dx $ is finite, the ratios of variance and bias in weighted-NN using $\phi(x,\gamma)$ and $\phi(x,0)$ become
		\begin{eqnarray*}
			&&\frac{\int_0^1 \phi(x,\gamma)^2x^{d-1}dx}{\left(\int_0^1 \phi(x,\gamma)x^{d-1}dx \right)^2}\frac{\left(\int_0^1 \phi(x,0)x^{d-1}dx\right)^2}{\left(\int_0^1\phi(x,0)^2 x^{d-1}dx\right)},\\\text{and }&& \frac{\left(\int_0^1 \phi(x,\gamma)x^{d-1}x^2dx\right)^2}{\left(\int_0^1 \phi(x,0)x^{d-1} x^2dx\right)^2}\frac{\left(\int_0^1 \phi(x,0)x^{d-1}dx\right)^2}{\left(\int_0^1 \phi(x,\gamma) x^{d-1}dx\right)^2}.
		\end{eqnarray*}
		In the context of $\phi(x,\gamma)=x^{-\gamma}$, the above two ratios refer to $s_{k,n,\gamma}^2/s_{k,n,0}^2$ and $t_{k,n,\gamma}^2/t_{k,n,0}^2$ in Theorem \ref{thm:main} respectively.
		
		For the ratio of variances, its gradient w.r.t $\gamma$ is 0 at $\gamma=0$:
		\begin{eqnarray*}
			\frac{\int_0^1 \phi(x,\gamma)^2x^{d-1}dx}{\left(\int_0^1 \phi(x,\gamma)x^{d-1}dx \right)^2}\frac{\left(\int_0^1 x^{d-1}dx\right)^2}{\left(\int_0^1 x^{d-1}dx\right)}-1=O(\gamma^2).
		\end{eqnarray*}
		However, for bias, it becomes
		\begin{eqnarray*}
			&&\frac{\left(\int_0^1 \phi(x,\gamma)x^2x^{d-1}dx\right)^2}{\left(\int_0^1 x^2x^{d-1}dx\right)^2}\frac{\left(\int_0^1 \phi(x,0)x^{d-1}dx\right)^2}{\left(\int_0^1 \phi(x,\gamma) x^{d-1}dx\right)^2}-1\\&=&O(\gamma^2)+\frac{2\gamma\Delta\left(\int_0^1 x^{d+1}dx\right)\left(\int_0^1 x^{d-1}dx\right)}{\left(\int_0^1 x^2x^{d-1}dx\right)^2\left(\int_0^1 \phi(x,\gamma) x^{d-1}dx\right)^2}.
		\end{eqnarray*}
		As a result, the increase of variance is of $O(\gamma^2)$, and the decrease of bias is a linear function of $\gamma$ since $\Delta<0$.
		
	\end{proof}

\end{document}